\documentclass{article}
\usepackage[final]{corl_2026}
\usepackage{amsmath,amssymb,amsfonts,amsthm}
\allowdisplaybreaks
\usepackage{graphicx}
\graphicspath{{./}{CoRL_2026_arxiv/}}
\usepackage{booktabs}
\usepackage{subcaption}
\usepackage{url}
\usepackage{enumitem}
\usepackage{titletoc}
\usepackage{algorithm}
\usepackage{algpseudocode}
\usepackage{etoc}
\usepackage{wrapfig}
\usepackage{caption}

\newif\ifshowcomments
\showcommentstrue

\newif\ifshowedits
\showeditsfalse

\newcommand{\myparagraph}[1]{%
  \vspace{1mm}\noindent\textbf{#1.}%
}

\newtheorem{assumption}{Assumption}
\newtheorem{definition}{Definition}
\newtheorem{lemma}{Lemma}
\newtheorem{proposition}{Proposition}
\newtheorem{theorem}{Theorem}
\newtheorem{corollary}{Corollary}
\newtheorem{remark}{Remark}

\newcommand{\X}{\mathcal X}

\newcommand{\mR}{\mathbb R}
\newcommand{\cB}{\mathcal B}
\newcommand{\cR}{\mathcal R}
\newcommand{\iid}{\mathrm{i.i.d.}}

\newcommand{\So}{\overline{S_0}}
\newcommand{\St}{\overline{S_T}}

\DeclareMathOperator{\dist}{dist}
\DeclareMathOperator{\reach}{reach}

\DeclareMathOperator{\supp}{supp}

\title{On the Limits of Sampling-Based Reachability: Geometry, Dynamics, and Sample Complexity}

\author{
  Jixian Liu$^{1}$ \qquad
  Ihab Tabbara$^{2}$ \qquad
  Hussein Sibai$^{2}$ \qquad
  Enrique Mallada$^{1}$ \\
  $^{1}$Department of Electrical and Computer Engineering,
  Johns Hopkins University \\
  $^{2}$Department of Computer Science,
  Washington University in St. Louis \\
  \texttt{\symbol{123}jliu376,mallada\symbol{125}@jhu.edu} \\
  \texttt{\symbol{123}i.k.tabbara,sibai\symbol{125}@wustl.edu}
}

\begin{document}
\maketitle

\makeatletter
\begingroup
\let\@footnotetext\H@@footnotetext
\renewcommand{\thefootnote}{}
\footnotetext{All codes and relevant materials are available at: \url{https://github.com/JeanLew01/reachapprox}}
\endgroup
\makeatother

\begin{abstract}
Reachability analysis is central to safety-critical control, robotics, and neural network verification, but classical computational methods, such as Hamilton--Jacobi reachability and set propagation, scale poorly with state dimension. Sampling-based methods have emerged as a promising alternative, often providing finite-sample guarantees that bound the probability-mass left uncovered. However, an explicit account of how the geometry of the initial set, the dynamics, and the sampling law affect the accuracy of the estimator is not fully available in the literature. We study this by casting sampling-based reachable-set recovery as geometric support estimation over a family of problems specified by an initial set, its dynamics, and a sampling law. First, we identify two regularity properties, positive reach of the initial set's complement and Lipschitz continuity of the dynamics, that together make recovery well-posed: a probability-mass coverage guarantee can be upgraded to accuracy $r$ in Hausdorff distance. Second, we bound the resulting sample complexity: recovery is achievable with $\tilde{\mathcal{O}}\big((e^{LT}/r)^n\big)$ samples, exponential in both the state dimension and the time horizon. Third, we show that neither can be removed: a minimax lower bound of $\tilde\Omega\big((e^{LT}/r)^n\big)$ holds for every estimator, so the exponential dependence on dimension and the degradation over the horizon are both intrinsic, not artifacts of a particular method. Experiments on nonlinear systems confirm that adversarial sampling improves constants but not the scaling.
\end{abstract}

\etocdepthtag.toc{main}

\section{Introduction}
\label{sec:intro}

Reachability analysis is a fundamental tool for certifying the safety of dynamical and learning-enabled systems~\citep{holmes2020reachable, meng2022learning, ivanov2021verisig, wu2024verified}. Classical computational methods, based on Hamilton--Jacobi reachability~\citep{bansal2017hamilton, chen2018hamilton} or set propagation~\citep{girard2005reachability, kurzhanski2000ellipsoidal}, give strong deterministic guarantees but scale poorly with state dimension. This barrier is especially pronounced in modern learning-enabled systems, where neural-network controllers and learned dynamics put classical methods out of reach. To improve scalability, sampling-based methods approximate reachable sets by propagating finitely many initial states through the dynamics~\citep{selim2022safe, ganai2023iterative, liu2025recurrent, ouyang2026symplectic}. They are model-agnostic, parallelizable, and easy to implement, applicable to general nonlinear and learning-based systems. These methods are part of a broader shift toward sampling-based control and optimization for high-dimensional systems~\citep{williams2018information, pan2025sampling, castellano2025data}, enabled by modern parallel hardware.

\begingroup
\spaceskip=\fontdimen2\font plus 1.5\fontdimen3\font minus \fontdimen4\font
The accompanying finite-sample guarantees, however, are limited. They either bound the probability-mass the estimate fails to cover~\citep{liebenwein2018sampling,devonport2021data,devonport2023data}, permitting an estimate to miss a thin, spatially separated region without violating the bound, or they measure error in Hausdorff distance against the convex hull of the reachable set~\citep{lew2021sampling,lew2022simple}, which over-approximates nonconvex sets by filling holes and merging disconnected components. What remains missing is a characterization of Hausdorff recovery of the actual set, and an account of how the sample cost depends on the geometry of the initial set, the dynamics, and the sampling law. Posed this way, the question becomes one of geometric support estimation under flow distortion, a perspective with a rich classical theory~\citep{devroye1980detection, cuevas2009set, korostelev1993minimax} that has not yet been brought to bear on the reachability setting.
\par\endgroup

\myparagraph{Contributions}
To fill this gap, we study autonomous systems $\dot x=F(x)$ and recover the reachable set $S_T=\cR_T(S_0)$ from endpoint samples of trajectories initialized in $S_0$. Viewing this as geometric support estimation under flow distortion, we study a problem family specified by the initial-set scale and geometry, the dynamics, and the sampling law.~\vspace{-1ex}
\begin{enumerate}[leftmargin=*, label=\arabic*., itemindent=0em, itemsep=0.1em]
    \item \textbf{Well-posedness.} We identify conditions under which probability-mass coverage can be upgraded to Hausdorff-accurate recovery: positive reach ($r_0$) of the initial set's complement, Lipschitz continuity ($L$) of dynamics, and a lower sampling density bound ($\rho$). We let $R$ denote the volume-equivalent radius of the initial set, i.e., the radius of an $n$-dimensional Euclidean ball with the same Lebesgue volume as $S_0$.

    \item \textbf{Sample complexity.} We prove a random-covering upper bound: $\tilde{\mathcal O}(
    (e^{LT}/r)^n)$
    endpoint samples suffice for $r$-accurate inner coverage. With an
    estimator-dependent outer-deviation condition, this yields a Hausdorff
    guarantee.

    \item \textbf{Unavoidability.} We prove a minimax lower bound: any estimator requires $\tilde \Omega((e^{LT}/r)^n)$ samples on some instance. Thus, the dependence on dimension, horizon, accuracy, and initial-set scale is intrinsic.

    \item \textbf{Numerical experiments.} On a 2D non-Lipschitz example and a
    multi-link MuJoCo robot arm, the empirical sample complexity tracks the
    predicted dimension-dependent scaling, while adversarial sampling improves
    constants but not scaling.
\end{enumerate}

\myparagraph{Closely related work}
Our analysis is closest to \citet{lew2022simple}, who derive a finite-sample Hausdorff guarantee for sampling-based approximation of the convex hull of the reachable set under related assumptions: $r$-convex initial geometry, Lipschitz dynamics, and a sampling density lower bound. Our
positive-reach assumption $\reach(S_0^c)\ge r_0$ (Definition~\ref{def:reach}) is stronger and implies their $r$-convexity. We go further in three directions: (i) Hausdorff recovery of the actual endpoint support beyond convex-hull approximations; (ii) explicit horizon dependence, rather than a single flow
Lipschitz constant; and (iii) a minimax lower bound that cannot be avoided. The classical $r^{-n}$ support-estimation rate~\citep{devroye1980detection,cuevas2009set,korostelev1993minimax,
ariascastro2019minimax} appears here through the scale-sensitive factor $(R/r)^n$ and is amplified by the flow factor $e^{nLT}$. Further related work, including scenario optimization, conformal prediction, and Lipschitz-based bloating, is discussed in Appendix~\ref{app:further_related_work}.

\myparagraph{Notation}
For a Lebesgue measurable set $A\subseteq\mR^n$, we write $A^c:=\mR^n\setminus A$ for its complement, $\overline A$ for its closure, and $\partial A$ for its boundary. Let $\lvert A\rvert:=\lambda_n(A)$, where $\lambda_n$ is the $n$-dimensional Lebesgue measure. For $x\in\mR^n$, $r>0$, and nonempty sets $A,B\subseteq\mR^n$, define $\cB_r(x):=\{y\in\mR^n:\|y-x\|\le r\}$, $d(x,A):=\inf_{a\in A}\|x-a\|$, and $\cB_r(A):=\{x\in\mR^n:d(x,A)\le r\}$. The Hausdorff distance between $A$ and $B$ is $d_H(A,B):=\max\{\sup_{a\in A}d(a,B),\ \sup_{b\in B}d(b,A)\}$. Let $\omega_n:=\lvert \cB_1(0)\rvert$ denote the volume of the unit ball in $\mR^n$. We use $A\oplus B:=\{a+b:a\in A,\ b\in B\}$, $A\ominus B:=\{x\in\mR^n:x+B\subseteq A\}$, and $A\circ B:=(A\ominus B)\oplus B$ for Minkowski sum, erosion, and opening, respectively. For a probability measure $P$, $\supp(P)$ denotes its support.


\section{Problem Formulation}
\label{sec:problem_formulation}

We consider an autonomous continuous-time dynamical system $\dot x = F(x)$, where $x\in\X\subseteq\mR^n$ and $F:\X\to\mR^n$. Given an initial condition $x_0\in\X$, let
$\varphi(t,x_0)$ denote the solution at time $t \ge 0$. For an initial set $S_0\subseteq\X$ and a time horizon $T>0$, the forward reachable set is
\begin{align}
    \cR_T(S_0)
    :=
    \{\varphi(T,x_0)\in\X:\ x_0\in S_0\}.
    \label{eq:frs_autonomous}
\end{align}
When $S_0$ and $T$ are known, we write $S_T:=\cR_T(S_0)$. This paper studies the problem of approximating $S_T$ from finitely many endpoint samples. Specifically, suppose that initial states $X_1,\ldots,X_N$ are sampled independently from a probability distribution $P_0$ supported on $\So$, and the corresponding endpoints are $Y_i := \varphi(T,X_i), i=1,\ldots,N$. These endpoints are $\iid$ samples from the pushforward distribution $P_T := (\varphi(T,\cdot))_{\#}P_0$, whose support is $\St$. A sampling-based reachable-set estimator is a measurable set-valued map $\widehat S_N = \widehat S_N(Y_1,\ldots,Y_N)\subseteq\mR^n$ that approximates the reachable set $S_T$. 

Many existing sampling-based and learning-based approaches provide probabilistic coverage guarantees. Informally, such results show that, with probability at least $1-\delta$, the estimated set misses at most $\varepsilon$ probability-mass under the endpoint distribution $P_T$. A typical guarantee takes the form $P_T(\St\setminus \widehat S_N)\le \varepsilon$ with probability larger than $1-\delta$ over the sampled endpoints and the joint distribution $P_T^N$ of $(Y_1,\ldots,Y_N)$~\citep{liebenwein2018sampling,devonport2023data}. However, probability-mass accuracy alone does not imply geometric accuracy. An estimator may miss only a small amount of probability-mass while still incurring a large Hausdorff error. This can happen when the missed region has small probability but is geometrically far from the estimated set, or when the reachable set contains thin cusps, spikes, or low-density regions. In safety and verification problems, such geometric errors are important: missing a small-probability but spatially significant region may still lead to an incorrect reachable-set certificate.

Motivated by this gap, our goal is to understand how many endpoint samples are necessary and sufficient for $\widehat S_N$ to approximate $S_T$ in Hausdorff distance. To be more specific, for a target accuracy $r>0$ and confidence level $1-\delta$, we seek conditions and sample size $N$ under which:


\begin{align}
    P_T^N
    \left(
        d_H(S_T,\widehat S_N)\le r
    \right)
    \ge 1-\delta .
    \label{eq:hausdorff_goal}
\end{align}


This objective is stronger than controlling the probability-mass of the missed region. It requires every point of the true reachable set to lie within distance $r$ of the estimator, and conversely requires the estimator not to extend too far away from the true reachable set. This property is arguably necessary for safety-critical circumstances.

The central question is therefore: when can a probability-mass guarantee be
converted into a Hausdorff-distance guarantee? Such a conversion is impossible without additional structure. Indeed, if there exists a point $x\in S_T$ such that $\mathcal B_r(x) \cap \St$ has arbitrarily small $P_T$-mass, then an estimator may miss this entire $r$-scale neighborhood while still satisfying a small probability error bound. In that case, the probability error is small, but the Hausdorff distance from $S_T$ to $\widehat S_N$ is at least of order $r$. Thus, to make $(\varepsilon,\delta)$-type probability guarantees geometrically meaningful, one needs a lower bound on the probability-mass of local geometric neighborhoods.

\begin{figure}[htbp]
    \centering
    \begin{subfigure}[t]{0.335\textwidth}
        \centering
        \includegraphics[width=\linewidth]{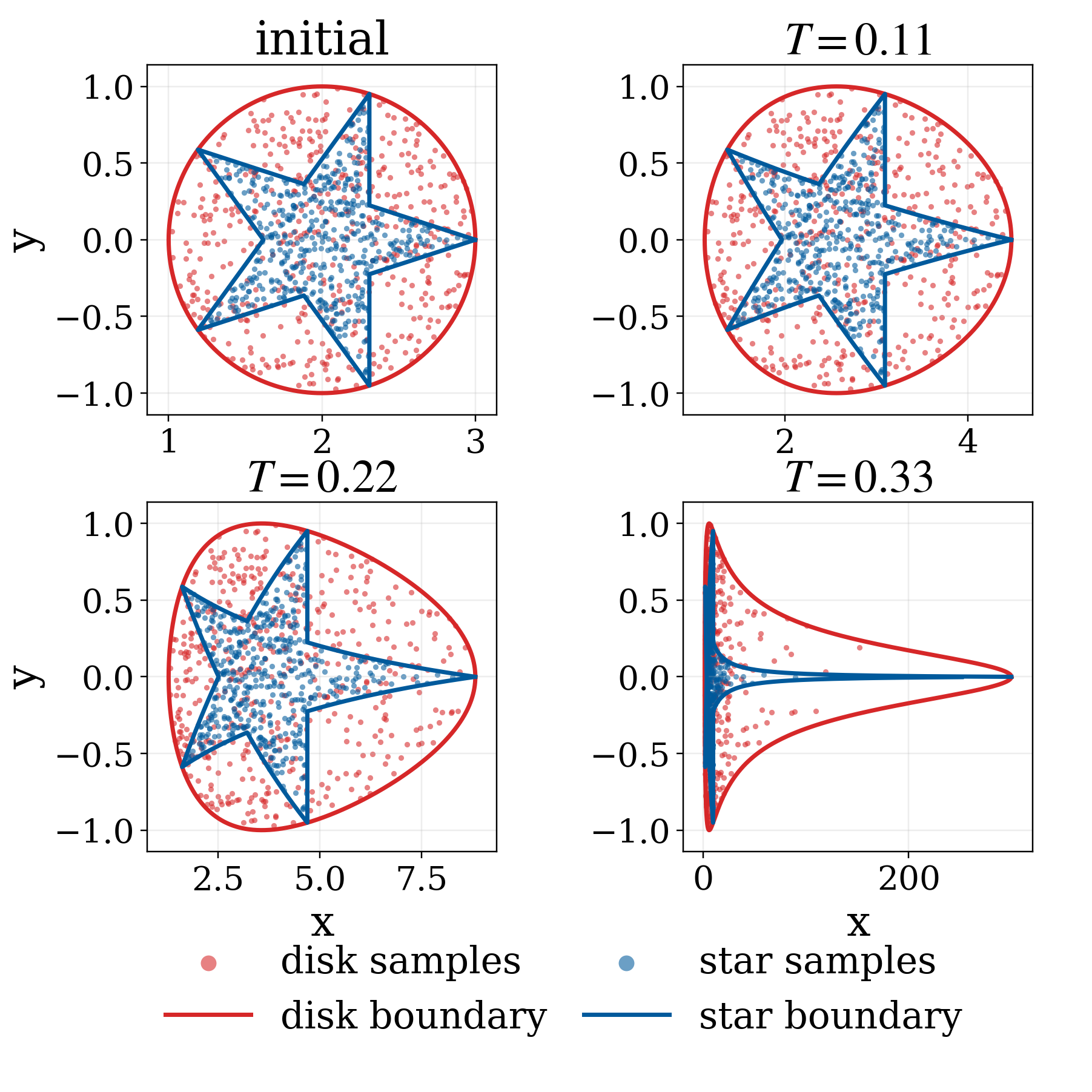}
        \label{fig:sample_flow_schematic}
    \end{subfigure}
    \hspace{-0.6em}
    \begin{subfigure}[t]{0.335\textwidth}
        \centering
        \includegraphics[width=\linewidth]{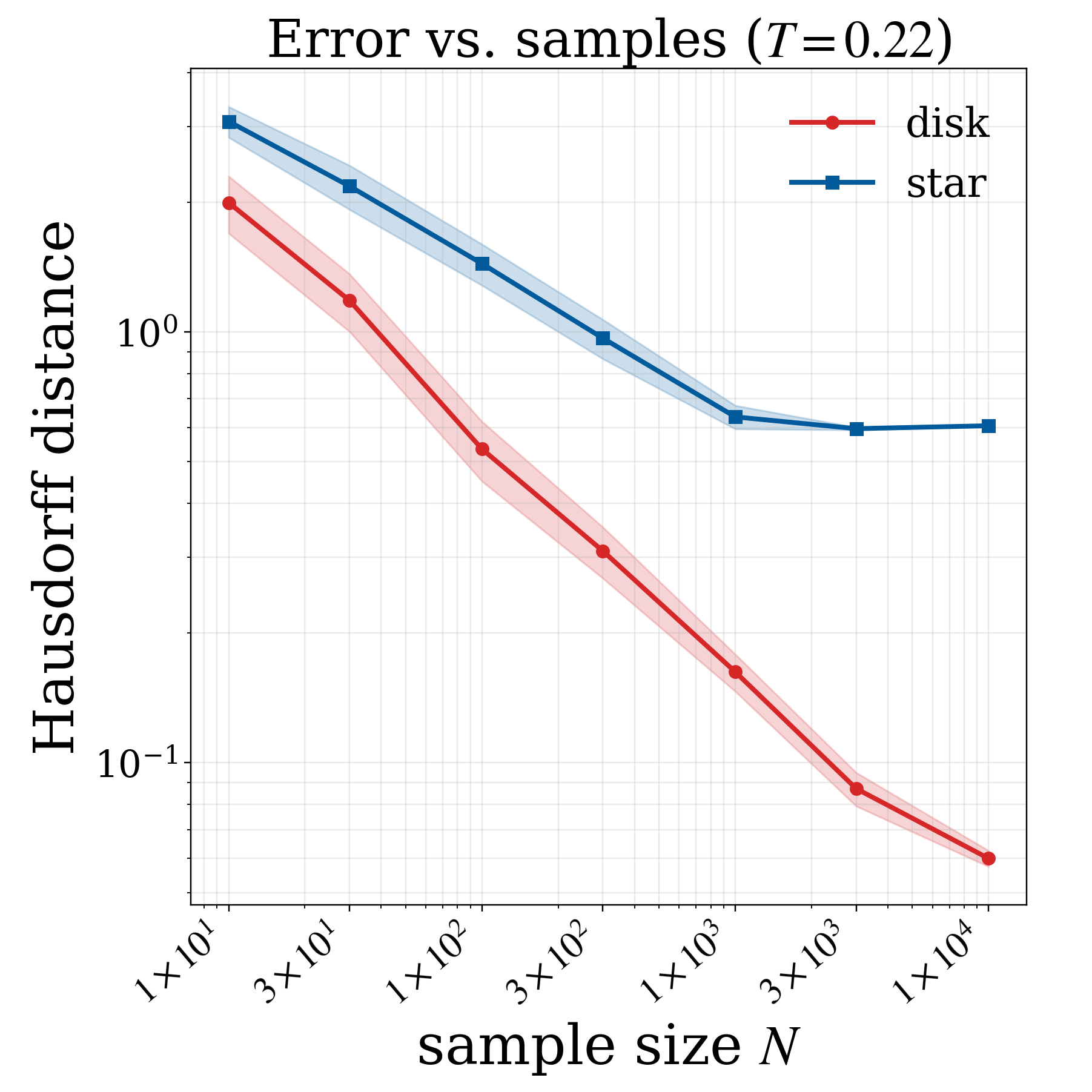}
        \label{fig:hausdorff_vs_time}
    \end{subfigure}
    \hspace{-0.6em}
    \begin{subfigure}[t]{0.335\textwidth}
        \centering
        \includegraphics[width=\linewidth]{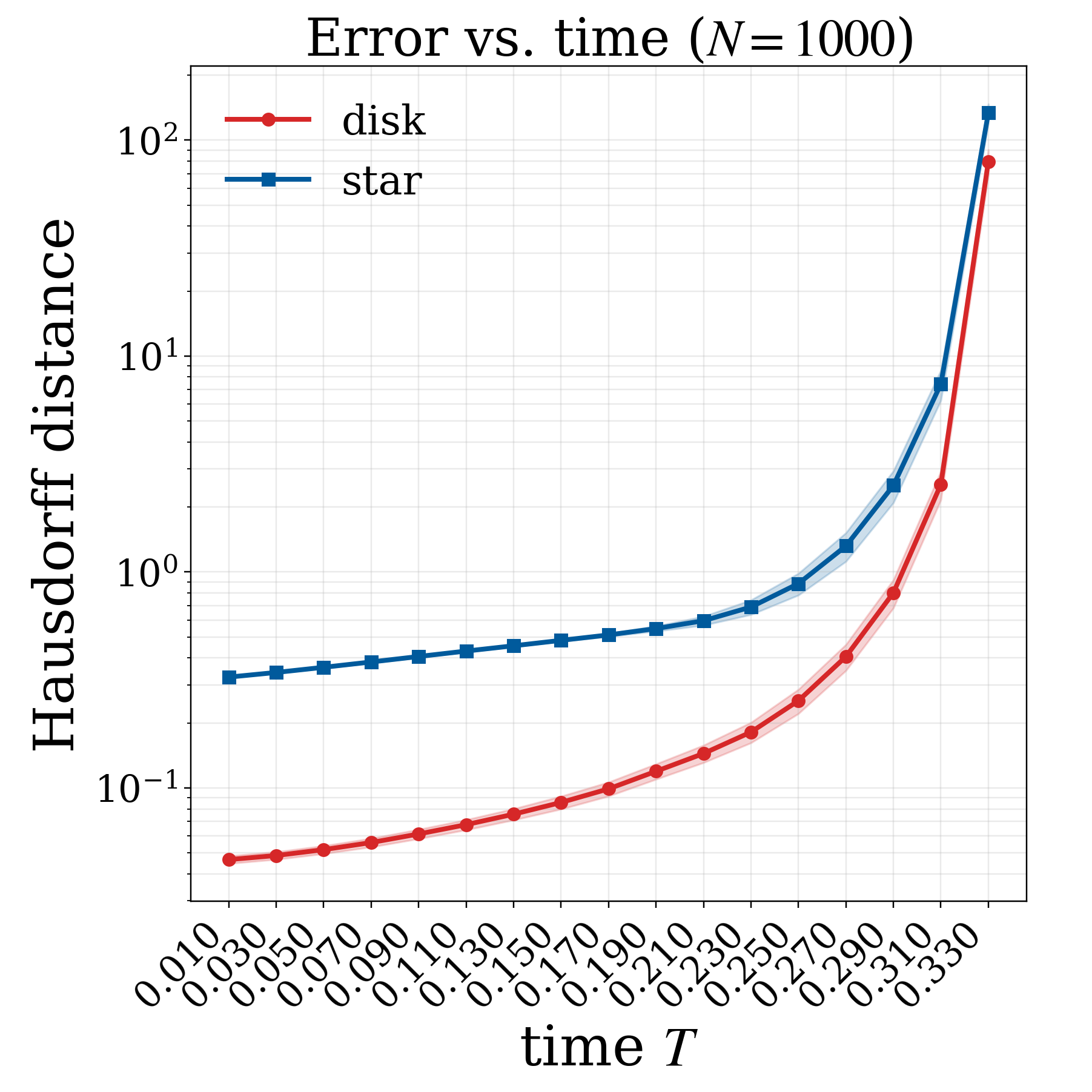}
        \label{fig:hausdorff_vs_samples}
    \end{subfigure}

    \caption{Reachable-set approximation under the autonomous dynamics
    $\dot y=0,\dot x=x^2$.}
    \label{fig:toy_example}
\end{figure}


Figure~\ref{fig:toy_example} illustrates these difficulties using the
non-globally-Lipschitz autonomous system
$\dot y=0,\ \dot x=x^2$. Starting from a disk and a nonconvex star of
comparable diameter, we uniformly sample initial conditions, propagate
them, and construct support estimators from the endpoints. At a fixed
horizon, the star incurs a larger Hausdorff error and exhibits pronounced
error saturation over the tested sample range, suggesting a
geometry-limited regime. At a fixed sampling budget, the flow
increasingly stretches neighborhoods at larger positive $x$, amplifying
sampling gaps and producing the super-exponential growth in Hausdorff
error observed over the plotted horizons. These trends reflect two coupled effects: the dynamics amplify local
coverage gaps through flow distortion, while the initial-set geometry
controls local volume and, together with the sampling density, the
probability mass contained in small neighborhoods. This motivates a
problem class combining dynamical, statistical, and geometric
regularity. As a standing well-posedness convention, for every finite
horizon $T$ considered, $\varphi(t,\cdot)$ is well defined and unique on
the relevant initial states for all $t\in[0,T]$. We further impose an $L$-Lipschitz condition on the vector field to
bound finite-horizon flow distortion, together with a lower bound on
the sampling density and geometric regularity of the initial set to
exclude thin degeneracies.

\begin{definition}[Reach~\citep{federer1959curvature}]
\label{def:reach}
Let $S\subset\mR^n$ be nonempty and closed. For $x\in\mR^n$, define
$d(x,S):=\inf_{y\in S}\|x-y\|$ and $\Pi_S(x):=\{y\in S:\|x-y\|=d(x,S)\}$. The reach of $S$ is defined as follows:
\begin{align*}
\reach(S) := \sup\left\{ r\ge 0: \Pi_S(x)\ \text{is a singleton whenever } d(x,S)<r \right\}.
\end{align*}
\end{definition}

We now collect the dynamical, geometric, and sampling requirements into a single problem family.

\begin{definition}[Problem family] \label{def:problem_family} Fix constants $R \ge r_0 >0, L>0, \rho\in(0,1]$. Let $\mathcal F_{R,L,r_0,\rho}$ denote the class of triples $(S_0,F,P_0)$ such that:
\begin{enumerate} 
\item $S_0\subset\X$ is a bounded open set with $\lvert S_0 \rvert=\omega_n R^n$; 
\item $\reach(S_0^c)\ge r_0$, where $S_0^c:=\mathbb R^n\setminus S_0$; 
\item $F:\X\to\mathbb R^n$ is $L$-Lipschitz on $\X$, i.e., $\|F(x)-F(y)\|\le L\|x-y\|, \forall x,y\in\X$; 
\item $P_0$ is absolutely continuous with respect to the Lebesgue measure on $\mathbb R^n$, with density $p_0$, and satisfies $\supp(P_0)=\overline{S_0}, p_0(x) \ge \frac{\rho}{\lvert S_0 \rvert} \text{ for a.e. } x\in S_0$. \end{enumerate} 
\end{definition}

\begingroup
\spaceskip=\fontdimen2\font plus 1.5\fontdimen3\font minus \fontdimen4{assumption}[Finite-horizon well-posedness and state containment]
\label{ass:finite-horizon-well-posedness}
For every instance $(S_0,F,P_0)\in\mathcal F_{R,L,r_0,\rho}$ and every
finite horizon $T$ under consideration, the flow
$\varphi(t,\cdot)$ is well defined and unique on $\overline{S_0}$ for
all $t\in[0,T]$. Moreover, $\varphi(t,\overline{S_0}) \subseteq \X, \forall t\in[0,T].$
\end{assumption}
\endgroup

The parameter $R$ fixes the global scale of the initial set through the volume constraint $|S_0|=\omega_n R^n$, where $R$ is the radius of an $n$-dimensional ball with the same volume as $S_0$. Thus, $R$ encodes the overall size of the support, while the reach parameter $r_0$ controls its local regularity. This family gives a local thickness property: by~\citep[Lemma~4.8]{rataj2019curvature}, $\overline{S}_0=\overline{S}_0\circ\mathcal B_{r_0}(0)$, and positive reach of the complement rules out inward cusps and thin degeneracies, yielding lower bounds of the form $|\mathcal B_s(x)\cap\So|\gtrsim s^n$ for $x\in\So$ and $0<s\le r_0$.\footnote{$A\gtrsim B$ means $A\ge cB$ for a constant $c>0$ independent of the scale of $B$.} This local volume bound is the bridge from probability to geometry. After accounting for Lipschitz flow distortion and endpoint density lower bounds, it ensures that missing an $r$-scale region of the reachable set incurs non-negligible probability mass, which enables Hausdorff accuracy guarantees.

\section{Reachable Set Approximation under Regular Geometry}
\label{sec:regular_geometry}

Under the problem family introduced above, we study Hausdorff approximation of the reachable set from finite endpoint samples. The key ingredient is a local endpoint probability lower bound, obtained by combining positive reach, the initial density lower bound, and Lipschitz flow distortion. This bound yields an inner-coverage sample-complexity result for any estimator containing all endpoint samples; a full Hausdorff guarantee follows after adding an estimator-dependent outer-deviation condition. We then prove that the $r^{-n}$ dependence is information-theoretically unavoidable.

\subsection{Local Probability Bounds under Flow Distortion}
\label{subsec:hausdorff_approximation}

We first establish a local volume lower bound for the initial support.
Positive reach prevents arbitrarily thin local regions, while the density
lower bound in Definition~\ref{def:problem_family} converts this geometric
thickness into sampling probability.

\begin{proposition}[Lower Lebesgue mass bound for the initial set]
\label{prop:lower-lebesgue-initial-set}
Let $(S_0,F,P_0)\in\mathcal F_{R,L,r_0,\rho}$. Then, for every
$x\in\So$ and every $r\in(0,r_0]$,
\begin{align*}
    \lvert\cB_r(x)\cap\So\rvert
    \ge
    \omega_n\,2^{-n}r^n.
\end{align*}
\end{proposition}

\begin{proof}
    See Appendix~\ref{app:lower-lebesgue-initial-set}.
\end{proof}




Combining Proposition~\ref{prop:lower-lebesgue-initial-set} with the initial density lower bound and the finite-time Lipschitz bound of the flow yields a uniform lower bound on local endpoint probability mass. Specifically, the preimage of an $r$-neighborhood centered on the reachable set contains an initial neighborhood of radius $e^{-LT}r$, whose probability mass is controlled by the local volume bound above.

\begin{corollary}[Endpoint local probability lower bound]
\label{cor:endpoint-local-probability-lower-bound}
Let $(S_0,F,P_0)\in\mathcal F_{R,L,r_0,\rho}$ and define $\Gamma_T(r) := \rho\,2^{-n}e^{-nLT}
    (\frac{r}{R})^n$. Then, for every $x\in\St$ and every $r\in(0,e^{LT}r_0], $
\begin{align*}
    P_T(\cB_r(x)\cap\St) \ge \Gamma_T(r).
\end{align*}
\end{corollary}

\begin{proof}
    See Appendix~\ref{app:flow-lipschitz-mass-propagation}.
\end{proof}

\subsection{Sample Complexity Bounds}
\label{subsec:sample_complexity}

Corollary~\ref{cor:endpoint-local-probability-lower-bound} shows that every $r$-neighborhood centered on the reachable set carries at least $\Gamma_T(r)$ endpoint probability mass. We now use this local probability lower bound to quantify how many endpoint samples are sufficient to cover the reachable set in the inner direction. The argument is estimator-agnostic: it only requires that the estimator contains all sampled endpoints.

\begin{theorem}[Sampling upper bound for inner coverage] 
\label{thm:sample-upper-bound-inner} 
Let $(S_0,F,P_0)\in\mathcal F_{R,L,r_0,\rho}$, and let $Y_1,\ldots,Y_N$ be $\iid$ samples from $P_T$. Suppose that the estimator $\widehat S_N$ contains all endpoint samples, i.e., $Y_i\in\widehat S_N$ for all $i=1,\ldots,N$. Fix $0<r\le 2e^{LT}r_0$, $\delta \in (0,1)$. If 
\begin{align*}
N \ge \frac{2^{2n}e^{nLT}R^n}{\rho r^n} \left[ \log\left( \frac{2^{3n}e^{nLT}R^n}{r^n} \right) + \log\frac{1}{\delta} \right], 
\end{align*}
then, with probability at least $1-\delta, \sup_{x\in \St} d(x,\widehat S_N)\le r$.
\end{theorem}

\begin{proof} 
See Appendix~\ref{app:sample_complexity}. 
\end{proof}

\begin{corollary}[Hausdorff bound with outer deviation] 
\label{cor:hausdorff-bound-outer-deviation} Under the conditions of Theorem~\ref{thm:sample-upper-bound-inner}, suppose in addition that the estimator satisfies the outer-deviation condition $\widehat S_N\subseteq \cB_{\eta}(S_T)$ for some $\eta \ge 0$. Then, with probability at least $1-\delta$, we have $d_H(S_T,\widehat S_N) \le \max\{r,\eta\}$.
\end{corollary} 
\begin{proof} 
See Appendix~\ref{app:sample_complexity}. \end{proof}

\begin{remark}[Estimator-dependent outer deviation] The parameter $\eta$ measures the outer deviation of the estimator, where $\widehat S_N\subseteq \cB_{\eta}(S_T)$ is equivalent to $\sup_{z\in\widehat S_N}d(z,S_T)\le\eta$. In Corollary~\ref{cor:hausdorff-bound-outer-deviation}, the sampling argument in Theorem~\ref{thm:sample-upper-bound-inner} controls the directed error from $S_T$ to $\widehat S_N$, while $\eta$ controls the opposite direction. Its value is estimator-dependent: $\eta=0$ for inner estimators, $\eta=h$ for a union-of-balls estimator with radius $h$, $\eta = \sup_{z\in\operatorname{conv}(S_T)}d(z,S_T)$ for a convex-hull estimator; and for a Christoffel-function estimator $\widehat S_N^{\rm Chr}:=\{z\in\mathcal X:\kappa_N(z)\le\tau_N\}$, the corresponding value is $\eta_{\rm Chr,N}:=\sup_{\kappa_N(z)\le\tau_N}d(z,S_T)$. 
\end{remark} 
The preceding upper bound is obtained through a random-covering argument. To determine whether its $r^{-n}$ dependence is merely a consequence of this argument or can be improved by a different estimator, we now establish a minimax lower bound. The proof constructs two admissible instances whose endpoint supports are separated by more than $2r$, while their $N$-sample distributions remain difficult to distinguish. Any estimator that is $r$-accurate on both instances would induce a reliable binary test. Consequently, below the stated sample threshold, every estimator must fail on at least one admissible instance with probability greater than $\delta$.

\begin{theorem}[Minimax Hausdorff lower bound]
\label{thm:minimax-hausdorff-lower-bound}
Suppose that $R\ge4r_0, \rho\in(0,1]$. Let $0<r \le 2^{-\frac{n+1}{n}}e^{LT}r_0, \delta\in\left(0,\frac12\right)$. If
\begin{align*}
    N < \frac{e^{nLT}R^n}{2^{n+1}r^n}
    \log\frac{1}{2\delta},
\end{align*}
then, for every measurable estimator $\widehat S_N$, based on $N$ $\iid$ endpoint samples, there exists an instance $(S_0,F,P_0) \in \mathcal F_{R,L,r_0,\rho}$ such that $P_T^N(d_H(\widehat S_N,\St)>r)>\delta.$
\end{theorem}

\begin{proof}
See Appendix~\ref{app:sample_complexity}.
\end{proof}

Together, the bounds show that the intrinsic dependence on the accuracy parameter is $r^{-n}$. The upper bound gives inner-coverage achievability up to logarithmic and flow-dependent factors, while the estimator-independent lower bound shows that $\Omega (\frac{e^{nLT}R^n}{2^{n+1}r^n} \log\frac{1}{2\delta})$ samples are unavoidable in the worst case. Table~\ref{tab:upper-lower-sample-complexity} summarizes the comparison. 
\vspace{-1em}
\begin{table}[htbp] 
\centering 
\caption{Comparison between the upper and lower bounds.}
\label{tab:upper-lower-sample-complexity} \begin{tabular}{ccc} 
\toprule & \textbf{Upper bound} & \textbf{Lower bound} \\
\midrule Type & Achievability & Minimax Hausdorff impossibility \\
Rate & $ \mathcal O\!\left( \frac{2^{2n}e^{nLT}R^n}{\rho r^n} \left[ \log\left( \frac{2^{3n}e^{nLT}R^n}{r^n} \right) + \log\frac{1}{\delta} \right] \right) $ & $ \Omega\!\left(\frac{e^{nLT}R^n}{2^{n+1}r^n} \log\frac{1}{2\delta} \right) $ \\
\bottomrule 
\end{tabular}
\end{table}

\vspace{-0.6em}

\section{Algorithm Design}
\label{sec:alg}

The $r^{-n}$ lower-bound dependence highlights the intrinsic coverage burden of Hausdorff reachable-set approximation and motivates an empirical question: under a fixed endpoint budget, can targeted initial-condition selection improve coverage over $\iid$ uniform sampling? To isolate the sampling effect, we hold the downstream estimator fixed within each comparison. At horizon $T$, the selected initial conditions generate $\mathcal Y_N:=\{Y_j\}_{j=1}^N, Y_j:=\varphi(T,X_j),\quad X_j\in S_0$, from which we construct $\widehat S_N=\mathrm C(\mathcal Y_N)$. We consider multiple choices of $\mathrm C$, but use the same estimator and endpoint budget when comparing sampling methods.

We compare $\iid$ uniform sampling with the adversarial sampling heuristic of \citet{lew2021sampling}. The latter adaptively updates initial conditions using the accumulated endpoint cloud to seek geometrically informative endpoints. Every propagated
endpoint is counted toward the total budget $N$. When $n_{\rm adv}=0$ in Algorithm~\ref{alg:adv_reachable_set_estimation}, no adversarial update is performed, and the procedure reduces to uniform sampling. Implementation details are provided in Appendix~\ref{app:adversarial-sampling}. Because the minimax lower bound is established for estimators based on $\iid$ endpoint samples, it does not directly apply to adaptive, statistically dependent samples. Our experiments therefore assess whether targeted sampling improves instance-wise finite-budget performance and whether dimension-dependent degradation persists, rather than establishing a minimax limitation for the adversarial one.

\section{Experiments}
\label{sec:exp}

In this section, we first quantitatively evaluate the finite-sample bounds on a globally Lipschitz system with $\iid$ uniform sampling. We then study dimension-dependent reachable-set approximation in a closed-loop robot-arm problem under both uniform and adversarial sampling. The robot-arm experiment is not intended to quantitatively validate the theoretical bounds; rather, it examines whether dimension-dependent degradation persists under targeted
sampling. The results show that adversarial sampling improves finite-sample efficiency but does not eliminate the intrinsic curse of dimensionality.

\subsection{Quantitative Validation of Sampling Bounds under Lipschitz Dynamics}
\label{subsec:lipschitz_sampling_bounds}

We consider the globally Lipschitz system $\dot{x}=x,\ \dot{y}=0$, whose flow is $\varphi_T(x,y)=(e^T x,y)$. We compare three initial sets centered at $(2,0)$ and normalized to have the same area $\pi$: a circle, a morphologically opened equilateral triangle, and an equilateral triangle. Their positive-reach parameters are $r_{0,\mathrm{circle}}=1$, $r_{0,\mathrm{open}}\approx0.309$, and $r_{0,\mathrm{triangle}}=0$, respectively. Thus, the first two geometries satisfy the positive-reach assumption, while the triangle is retained only as an empirical comparison outside this assumption. In each trial, we draw $N$ i.i.d.\ uniform samples, propagate them through $\varphi_T$, and estimate the reachable set by their convex hull. The error is the approximate symmetric Hausdorff distance between the estimated boundary and a high-resolution reference boundary. The time sweep fixes $N=1000$ and varies $T$ from $0.01$ to $2$, while the sample-size sweep fixes $T=1$ and uses $N\in\{3\times 10^2,10^3,3\times 10^3,10^4,3\times 10^4,10^5,3\times 10^5\}$. For the theoretical comparison, we set $\delta=0.05$ and use the common reach $r_0=\min\{r_{0,\mathrm{circle}},r_{0,\mathrm{open}}\}\approx0.309$.

\begin{figure}[htbp]
    \centering
    \includegraphics[width=\linewidth]{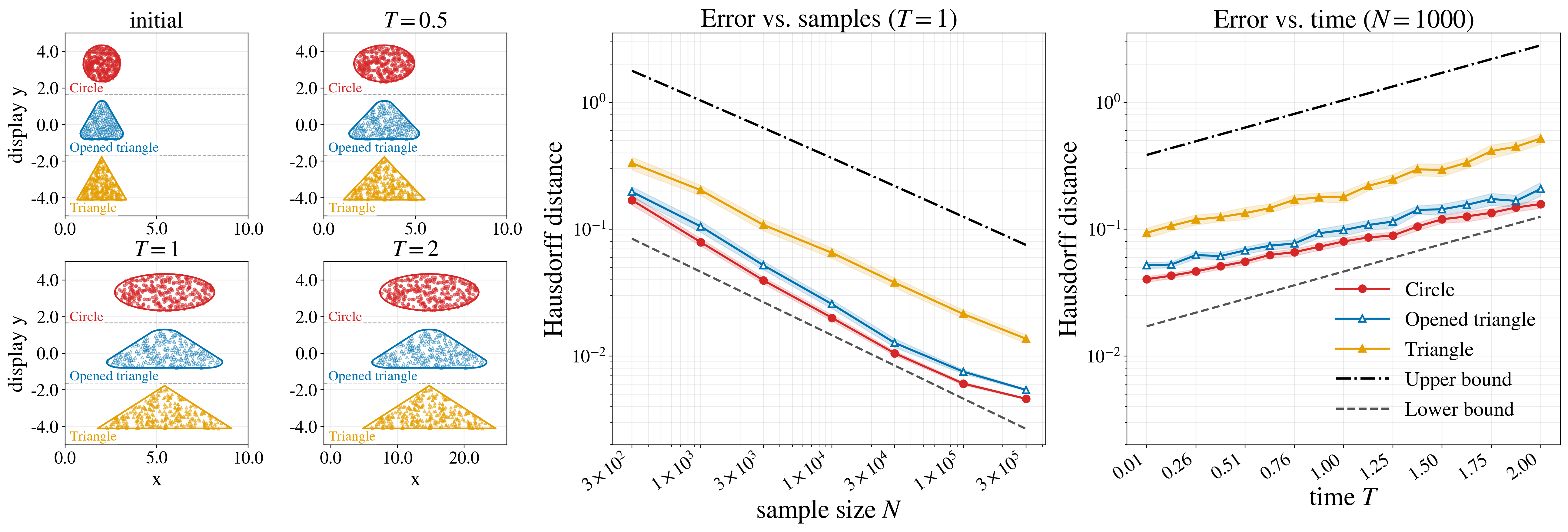}
    \caption{Reachable-set approximation for $\dot{x}=x,\dot{y}=0$. Left: snapshots at $T\in\{0,0.5,1,2\}$ with display-only vertical offsets; middle: Hausdorff error versus time at $N=1000$; right: error versus sample size at $T=1$. Shading denotes 95\% confidence intervals over 50 trials.}
    \label{fig:lipschitz_quantitative_validation}
\end{figure}

The empirical error increases with time as the flow expands in the $x$-direction and decreases as the sample size grows. At small sample budgets,
additional samples rapidly eliminate large uncovered regions, whereas at larger budgets the Hausdorff error is governed by the largest remaining boundary gaps, resulting in diminishing returns. In the time sweep, the nonexpanding $y$-direction initially moderates error growth, while the $e^T$ expansion in the $x$-direction eventually dominates, making the empirical
curves approximately parallel to the theoretical lower-bound scale. The triangle generally has the largest error, whereas the circle has the smallest. At fixed $T$, the upper and minimax lower scales behave as $((\log N)/N)^{1/n}$ and $N^{-1/n}$, respectively. For $n=2$, both have
leading exponent $-1/2$, up to the logarithmic factor in the upper bound. Overall, the empirical trends support the theoretical sampling upper and
minimax lower bounds.

\subsection{Intrinsic Curse of Dimensionality Cannot be Circumvented}
We also consider a vertical $n$-link robot arm tracking problem with $n\in\{2,3,4\}$. The state is $x=[q^\top,v^\top]^\top\in\mR^{2n}$, where $q\in\mR^n$ denotes joint angles and $v=\dot q\in\mR^n$ denotes joint velocities. The trajectories are simulated in MuJoCo under the rigid-body dynamics $ M(q)\dot v + C(q,v)v + g(q) = \tau$. Given a smooth reference trajectory $q_d(t)$, we use a non-adaptive tracking controller $\tau = M(q)\ddot q_d(t)+C(q,v)v+g(q) -K_p(q-q_d(t))-K_d(v-\dot q_d(t))$ for control. As shown in Figure~\ref{fig:robot_arm_curse} and Table~\ref{tab:robotarm-slopes}, the Hausdorff error decreases with the sample size under both uniform and adversarial sampling. However, the decay becomes slower as the state dimension increases: the fitted log--log slopes become less negative from dimension $4$ (2-link) to dimension $8$ (4-link). Adversarial sampling gives steeper slopes than uniform sampling in all dimensions, indicating better finite-sample efficiency. Nevertheless, its slopes also flatten as the dimension increases. Thus, targeted sampling improves the constants but does not remove the intrinsic dimension-dependent degradation predicted by the $r^{-n}$-type sample-complexity scaling. Additional implementation details, more detailed comparisons, and explanations, including the sample-budget-dependent improvement of adversarial sampling, are provided in Appendix~\ref{app:additional_results}.

\begin{minipage}[t]{0.00\textwidth}
\end{minipage}%
\hfill%
\begin{minipage}[t]{1.0\textwidth}
\vspace{-2.5em}
\begin{figure}[H]
    \centering
    \begin{subfigure}[t]{0.49\textwidth}
        \centering
        \includegraphics[width=\textwidth]{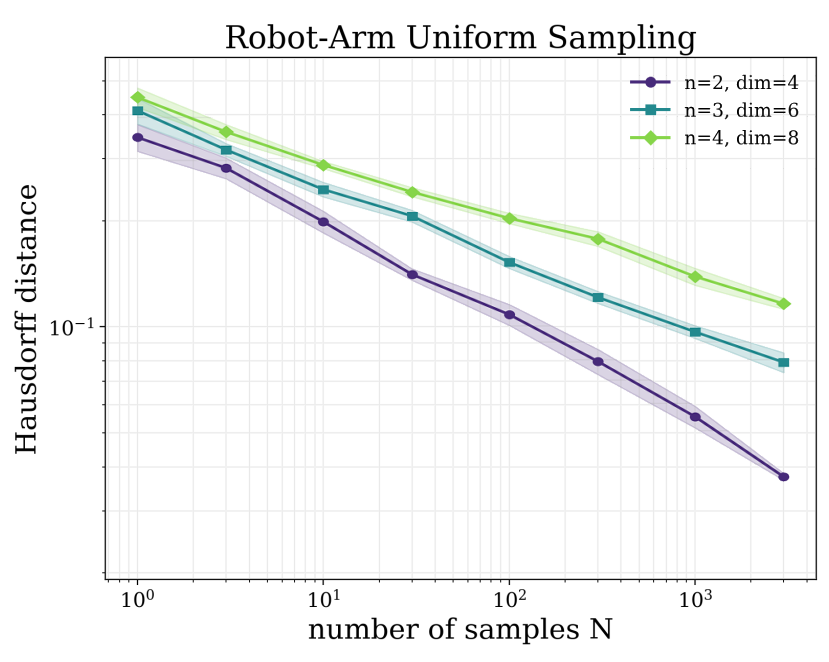}
        \label{fig:kura_uni}
    \end{subfigure}
    \begin{subfigure}[t]{0.49\textwidth}
        \centering
        \includegraphics[width=\textwidth]{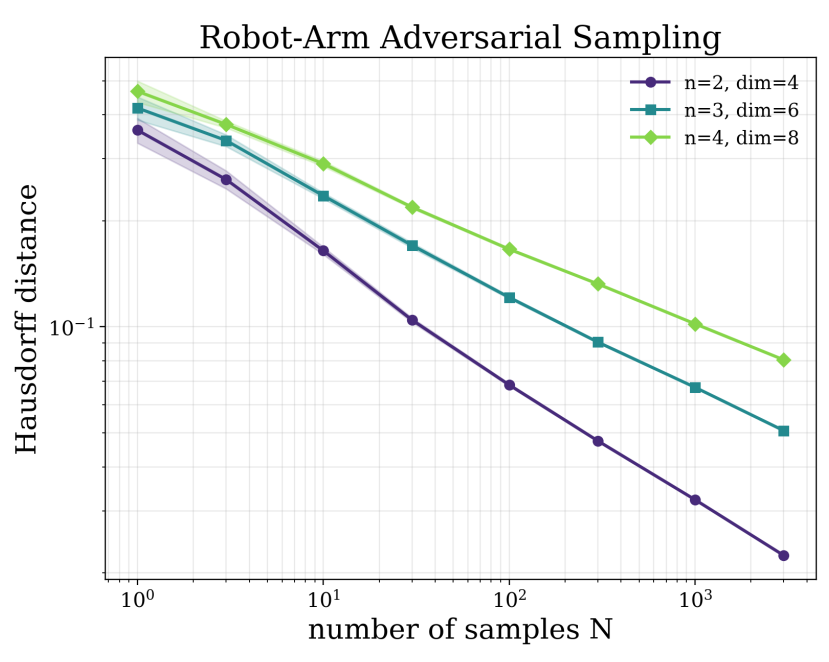}
        \label{fig:kura_adv}
    \end{subfigure}
    \vspace{-0.9em}
    \caption{Reachable-set approximation for the closed-loop robot arm. Left: Hausdorff error versus sample size under uniform sampling; right: error versus sample size under adversarial sampling. Curves correspond to state dimensions $4$, $6$, and $8$ (2-, 3-, and 4-link arms, respectively).}
    \label{fig:robot_arm_curse}
\end{figure}


\centering
\begin{tabular}{cccc}
\toprule
Dimension & 4 & 6 & 8\\
\midrule
Uniform sampling
& -0.2806 & -0.2094 & -0.1662 \\
Adversarial sampling
& -0.3535 & -0.2701 & -0.2224\\
\bottomrule
\end{tabular}
\captionsetup{hypcap=false}
\vspace{-.1em}
\captionof{table}{Empirical log--log slopes for robot-arm uncertainty propagation. Each slope is obtained by fitting $\log d_H$ versus $\log N$, where $d_H$ is the Hausdorff distance and $N$ is the number of samples.}
\label{tab:robotarm-slopes}
\vspace{-1.5ex}
\end{minipage}

\section{Conclusion and Future Work}
\label{sec:conclusion}

This paper characterized finite-sample reachable-set recovery in Hausdorff distance by casting endpoint sampling as geometric support estimation under a pushforward distribution. Under positive reach of the initial-domain complement, Lipschitz dynamics, and a uniform lower bound on the initial density, we converted local probability coverage into Hausdorff control, subject to an estimator-dependent outer-deviation condition. This separates the sampling requirement for inner coverage from the estimator-specific control of outer deviation. The resulting upper and minimax lower bounds match up to logarithmic factors, identifying $e^{nLT}r^{-n}$ as the intrinsic worst-case sampling burden. Experiments support the predicted $\iid$ trends and show that adversarial sampling improves finite-budget performance while
dimension-dependent degradation persists. Future work could develop structure-aware alternatives to these worst-case assumptions using
physics-informed inductive biases~\citep{raissi2019physics}, including energy conservation~\citep{duong21hamiltonian}, contraction~\citep{singh2023robust}, and Hamiltonian structure~\citep{greydanus2019hamiltonian}, potentially reducing the effective
sample complexity of physically structured systems~\citep{djeumou23how}.

\section{Limitations}
The analysis is restricted to autonomous dynamics. This includes uncontrolled systems and closed-loop systems, but not open-loop reachable sets, which are unions over a control class rather than images of $S_0$ under a single flow map. As a result, the injectivity, volume-distortion, and endpoint-density arguments used here do not directly apply. Moreover, endpoint mass depends on how control signals are sampled, so geometrically important regions may receive very small probability mass.

\clearpage
\acknowledgments{This work was supported by the NSF Global Centers program under Grant No.~2330450 and by the DOE Office of Science (ASCR) under Award No.~826565.}

\bibliography{example}

\clearpage
\appendix
\raggedbottom
\etocdepthtag.toc{appendix}

\section*{Appendix}

The Appendix is organized as follows. Section~\ref{app:further_related_work} provides additional related work on sampling-based reachable set approximation methods. Section~\ref{app:proofs} presents the proofs of the theoretical results in the main paper, including the lower Lebesgue-mass bound, the mass-to-Hausdorff conversion argument, and the upper and lower sample-complexity bounds. Section~\ref{app:additional_experiments} reports additional numerical experiments and implementation details.

\tableofcontents
\clearpage

\section{Further Related Work}
\label{app:further_related_work}

Beyond the Hausdorff-oriented works discussed in the main paper, much of the sampling-based reachability literature studies weaker probabilistic notions of approximation, such as volume coverage and violation probability, which are often characterized by scenario optimization and conformal prediction. We briefly review these adjacent frameworks below. While some of them provide important finite-sample and high-confidence guarantees, they do not by themselves ensure geometric recovery of the actual reachable set. This makes them complementary to our goal of understanding when probability-level guarantees can be upgraded to Hausdorff accuracy.

\subsection{Other Sampling-based Reachable Set Approximation Methods}

\paragraph{Scenario optimization and sampling-based certificates.}
Scenario optimization provides a principled framework for deriving probabilistic guarantees from finitely many sampled constraints. In the classical scenario approach, an infinite family of constraints is replaced by a finite set of sampled constraints, leading to high-confidence bounds on the probability of constraint violation~\citep{calafiore2006scenario,campi2008exact}. This viewpoint has been adopted in reachability and safety verification, where sampled reachable states are used to construct data-driven outer approximations or safety certificates. For instance, scenario-based methods have been used to estimate reachable sets within tractable templates such as boxes, ellipsoids, and other parameterized set families~\citep{devonport2020estimating,dietrich2025data}. Related work has also studied reach-avoid problems for discrete-time linear time-invariant systems with additive probabilistic uncertainty, using Voronoi partition-based scenario reduction to accelerate sampling-based stochastic reachability computation~\citep{sartipizadeh2019voronoi}. Scenario-based techniques have also been extended to nonconvex set representations~\citep{dietrich2024nonconvex} and model-predictive control settings~\citep{hewing2019scenario}. These methods are attractive since they reduce reachable-set estimation or safety certification to finite-dimensional optimization problems with explicit confidence bounds. However, their guarantees are typically expressed in terms of violation probability or marginal coverage, rather than Hausdorff distance. This distinction is important: a set of small probability-mass can still be geometrically far from the estimator, especially in the presence of low-density regions, thin components, or separated reachable branches. Thus, such guarantees do not directly imply Hausdorff accuracy without additional geometric regularity, such as positive reach and lower bounds on local endpoint probability mass.

\paragraph{Set growth and Lipschitz constant-based bloating.}
Another line of sampling-based reachability methods estimates how sets grow under the dynamics by bounding the sensitivity of trajectories to initial conditions, disturbances, or model parameters. These methods often rely on Lipschitz constants, local discrepancy functions, finite-order Taylor approximations, or learned expansion rates to inflate sampled trajectories into tubes covering nearby trajectories. For example, some algorithms use smoothness of the dynamics together with Lipschitz constants to control the Lagrange remainder of Taylor approximations and thereby compute outer approximations of reachable sets~\citep{koller2018learning}. Tools such as DryVR use discrepancy functions to quantify how the distance between trajectories evolves over time, allowing finitely many simulations to certify a larger reach tube~\citep{fan2017dryvr}. More recent statistical verification methods, such as GoTube, estimate local expansion or Lipschitz-like quantities to construct probabilistic tubes for continuous-depth models and neural ODEs~\citep{gruenbacher2022gotube,chen2018neural}. A key limitation of this line of work is that tight Lipschitz constants are rarely available in data-driven or learning-based control systems. Computing upper bounds on the Lipschitz constants of neural networks is itself an active research problem~\citep{fazlyab2019efficient,latorre2020lipschitz}, and such bounds can be computationally expensive or overly conservative. Although scalable sampling-based procedures can estimate Lipschitz-like quantities~\citep{dean2020robust}, these estimates are generally not guaranteed to be valid upper bounds. These approaches are closely related to our flow-dependent analysis, since both use sensitivity of the dynamics to translate finite sampling resolution into geometric coverage after propagation. Our theory instead directly characterizes how many endpoint samples are needed to approximate the actual reachable set in Hausdorff distance, and how this sample complexity depends on the set scale, target accuracy, state dimension, and flow distortion.

\paragraph{Conformal prediction and distribution-free coverage.} Closely related to scenario optimization, conformal prediction provides
distribution-free finite-sample marginal coverage guarantees under
exchangeability. Given a calibration dataset and a user-defined nonconformity score, conformal methods compute an empirical quantile of the calibration scores to construct prediction sets with distribution-free marginal coverage, where a future test sample is contained in the conformal set with a prescribed probability~\citep{shafer2008tutorial}. Since this construction does not require a parametric model of the data-generating distribution, it has been widely used as an uncertainty-quantification tool in learning-enabled autonomous systems, including perception, prediction, safe control, offline verification, and online monitoring~\citep{lindemann2025formal,mei2026perceive,yao2025towards,tabbara2025statistically}. For instance, in safe planning for dynamic environments, conformal prediction calibrates prediction regions around learned trajectory predictors and incorporates them into model predictive control constraints to obtain user-specified probabilistic safety guarantees~\citep{lindemann2023safe}. In reachability and safety verification, conformal prediction is typically used as a calibration layer on top of a data-driven reachable-set representation or a learned reachability surrogate. It has been used to calibrate Christoffel function's sublevel set thresholds for data-driven reachability, improving sample efficiency and robustness to outliers~\citep{tebjou2023data}; to verify DeepReach-style~\citep{bansal2021deepreach} neural HJ reachable tubes with probabilistic safety guarantees despite outlier errors~\citep{lin2024verification}; and to inflate neural surrogate flowpipes for black-box stochastic systems using conformal residual bounds~\citep{hashemi2023data}. Recent PCA-based variants further reduce the conservatism of this inflation and improve scalability under distribution shift~\citep{hashemi2025pca}. These methods are appealing because they require weak distributional assumptions and provide finite-sample coverage or safety guarantees. However, their guarantees remain probabilistic: a conformal set contains a future sample, trajectory, or prediction error with prescribed probability, but need not be close to the full reachable set in Hausdorff distance. It may cover most endpoint samples while missing a low-probability but spatially separated component. Thus, conformal prediction is complementary to our analysis, which identifies geometric, dynamical, and density conditions under which probability-mass coverage can be upgraded to Hausdorff accuracy.

\section{Proofs and Auxiliary Results}
\label{app:proofs}

\subsection{Lower Lebesgue Mass Bound for the Initial Set}
\label{app:lower-lebesgue-initial-set}

We prove the local thickness property used in
Proposition~\ref{prop:lower-lebesgue-initial-set}. Recall that the positive-reach condition is imposed on the closed complement $S_0^c=\mR^n\setminus S_0$.

\noindent\textbf{Proposition~\ref{prop:lower-lebesgue-initial-set}}
(Lower Lebesgue mass bound for the initial set).
Let $(S_0,F,P_0)\in\mathcal F_{R,L,r_0,\rho}$. Then, for every $x\in\So$ and every $r\in(0,r_0]$,
\begin{align*}
    \lvert \cB_r(x)\cap \So\rvert
    \ge
    \omega_n\,2^{-n}r^n .
\end{align*}

\begin{figure}[htbp]
    \centering
    \includegraphics[width=0.9\linewidth]{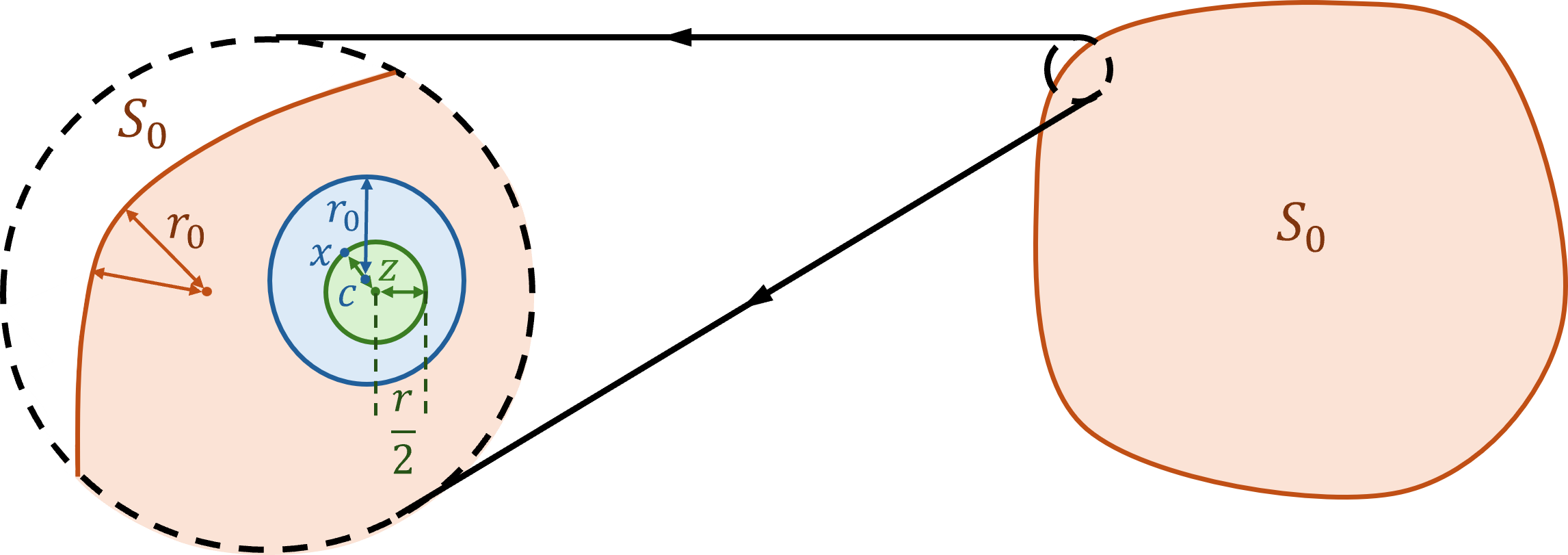}
    \caption{Local ball inclusion used in the proof of
    Proposition~\ref{prop:lower-lebesgue-initial-set}.}
    \label{fig:cover}
\end{figure}

\begin{proof}[Proof of Proposition~\ref{prop:lower-lebesgue-initial-set}]
Fix $x\in\So$ and $r\in(0,r_0]$. By the positive-reach opening property \citep[Lemma~4.8]{rataj2019curvature}, the condition $\reach(S_0^c)\ge r_0$ implies that $\So$ has an interior rolling-ball representation at scale $r_0$: for every $x\in\So$, there exists $c\in\So$ such that $x\in \cB_{r_0}(c), \cB_{r_0}(c)\subseteq \So$. We claim that $\cB_r(x)\cap \So$ contains a Euclidean ball of radius $r/2$. This immediately implies the
desired lower bound. Let $d:=\|x-c\|$. If $x=c$, then since $r\le r_0$,
\begin{align*}
    \cB_r(x)\subseteq \cB_{r_0}(c)\subseteq \So .
\end{align*}
Therefore,
\begin{align*}
    \lvert \cB_r(x)\cap\So\rvert
    =
    \omega_nr^n
    \ge
    \omega_n 2^{-n}r^n .
\end{align*}

Now suppose $x\neq c$. Define $\nu:=\frac{c-x}{\|c-x\|}, z:=x+\frac{r}{2}\nu$. Then $\|z-x\|=r/2$, and hence $\cB_{r/2}(z)\subseteq \cB_r(x)$. It remains to show that $\cB_{r/2}(z)\subseteq \cB_{r_0}(c)$. Take any $w\in\cB_{r/2}(z)$. If $d\ge r/2$, then $\|z-c\|=d-r/2$, and therefore
\begin{align*}
    \|w-c\|
    \le
    \|w-z\|+\|z-c\|
    \le
    \frac{r}{2}+\left(d-\frac{r}{2}\right)
    =
    d
    \le
    r_0 .
\end{align*}
If $d<r/2$, then $\|z-c\|=r/2-d$, and hence
\begin{align*}
    \|w-c\|
    \le
    \|w-z\|+\|z-c\|
    \le
    \frac{r}{2}+\left(\frac{r}{2}-d\right)
    =
    r-d
    \le
    r
    \le
    r_0 .
\end{align*}
Thus, in both cases, $w\in \cB_{r_0}(c)$. Since $w$ was arbitrary,
\begin{align*}
    \cB_{r/2}(z)
    \subseteq
    \cB_r(x)\cap \cB_{r_0}(c)
    \subseteq
    \cB_r(x)\cap \So .
\end{align*}
Taking Lebesgue measure gives
\begin{align*}
    \lvert \cB_r(x)\cap \So\rvert
    \ge
    \lvert \cB_{r/2}(z)\rvert
    =
    \omega_n\left(\frac{r}{2}\right)^n
    =
    \omega_n\,2^{-n}r^n .
\end{align*}
This proves the claim.
\end{proof}

\subsection{Closure of the Reachable Set}
\label{app:closure-reachable-set}

We first record a simple topological fact used to identify the support of the endpoint distribution.

\begin{lemma}[Reachability commutes with closure]
\label{lem:flow-image-closure}
Suppose $S_0\subset\mR^n$ is bounded and the finite-time flow map $\varphi(T,\cdot)$ is continuous on $\So$. Then
\begin{align*}
    \cR_T(\So) = \varphi(T,\So) = \overline{\varphi(T,S_0)}= \overline{\cR_T(S_0)} .
\end{align*}
In particular, if $S_T:=\cR_T(S_0)$, then $\cR_T(\So)=\St$.
\end{lemma}

\begin{proof}
Since $S_0$ is dense in $\So$ and $\varphi(T,\cdot)$ is continuous on $\So$,
we have $\varphi(T,\So) \subseteq \overline{\varphi(T,S_0)}$. Conversely, since $S_0$ is bounded, $\So$ is compact. Therefore $\varphi(T,\So)$ is compact, hence closed. Moreover, $\varphi(T,S_0)\subseteq \varphi(T,\So)$, so
\begin{align*}
    \overline{\varphi(T,S_0)}
    \subseteq
    \varphi(T,\So).
\end{align*}
Combining the two inclusions gives the claim.
\end{proof}

\subsection{Flow Lipschitzness and Local Probability-Mass Propagation}
\label{app:flow-lipschitz-mass-propagation}

We prove the finite-time flow estimate and endpoint local probability bound used in the main text in this subsection. Throughout this subsection, $\St=\cR_T(\So)=\overline{\cR_T(S_0)}$, where the last equality follows from Lemma~\ref{lem:flow-image-closure}. The argument requires only the upper Lipschitz bound of the finite-time flow; volume distortion, inverse-flow regularity, and propagated endpoint-density bounds are not needed.

\begin{lemma}[Finite-time Lipschitz bound]
\label{lem:flow-upper-lipschitz}
Let $(S_0,F,P_0)\in\mathcal F_{R,L,r_0,\rho}$. Then, for every $x,y\in\So$,
\begin{align}
    \|\varphi(T,x)-\varphi(T,y)\|
    \le
    e^{LT}\|x-y\|.
    \label{eq:flow-upper-lipschitz}
\end{align}
\end{lemma}

\begin{proof}
Fix $x,y\in\So$. Since
$(S_0,F,P_0)\in\mathcal F_{R,L,r_0,\rho}$, the vector field $F$ is
$L$-Lipschitz on $\X$. The integral form of the flow gives
\begin{align*}
    \varphi(s,x)-\varphi(s,y)
    =
    x-y
    +
    \int_0^s
    \left[
        F(\varphi(\sigma,x))
        -
        F(\varphi(\sigma,y))
    \right]d\sigma .
\end{align*}
Taking norms and using the Lipschitz continuity of $F$, we obtain
\begin{align*}
    \|\varphi(s,x)-\varphi(s,y)\|
    \le
    \|x-y\|
    +
    L\int_0^s
    \|\varphi(\sigma,x)-\varphi(\sigma,y)\|\,d\sigma .
\end{align*}
Gronwall's inequality~\citep[Corollary~3.17]{fb2024book} yields
\begin{align*}
    \|\varphi(s,x)-\varphi(s,y)\|
    \le
    e^{Ls}\|x-y\|,
    \qquad
    s\in[0,T].
\end{align*}
Taking $s=T$ proves \eqref{eq:flow-upper-lipschitz}.
\end{proof}

We now transfer the initial local probability bound directly to the endpoint distribution through the preimage of a local ball. This preserves the coupling between the geometry and sampling density induced by the flow and avoids the conservatism introduced by bounding propagated
volume and endpoint density separately.

\noindent\textbf{Corollary~\ref{cor:endpoint-local-probability-lower-bound}}
(Endpoint local probability lower bound).
Let $(S_0,F,P_0)\in\mathcal F_{R,L,r_0,\rho}$, and define
\begin{align*}
    \Gamma_T(r)
    :=
    \rho\,2^{-n}e^{-nLT}
    \left(\frac{r}{R}\right)^n.
\end{align*}
Then, for every $x\in\St$ and every
$r\in(0,e^{LT}r_0]$,
\begin{align*}
    P_T\bigl(\cB_r(x)\cap\St\bigr)
    \ge
    \Gamma_T(r).
\end{align*}

\begin{proof}[
Proof of Corollary~\ref{cor:endpoint-local-probability-lower-bound}]
Fix $x\in\St$ and $r\in(0,e^{LT}r_0]$. Since
$\St=\varphi(T,\So)$, choose $y\in\So$ such that $x=\varphi(T,y).$ Set $s:=e^{-LT}r$. The assumption $r\le e^{LT}r_0$ implies that $s\in(0,r_0]$. For every $z\in\cB_s(y)\cap\So$, Lemma~\ref{lem:flow-upper-lipschitz} gives
\begin{align*}
    \|\varphi(T,z)-x\|
    &=
    \|\varphi(T,z)-\varphi(T,y)\|
    \\
    &\le
    e^{LT}\|z-y\|
    \\
    &\le
    e^{LT}s
    \\
    &=
    r.
\end{align*}
Moreover, since $z\in\So$, we have
$\varphi(T,z)\in\St$. Therefore,
\begin{align*}
    \cB_s(y)\cap\So
    \subseteq
    \varphi(T,\cdot)^{-1}
    \bigl(\cB_r(x)\cap\St\bigr).
\end{align*}

Using the pushforward identity and the lower bound of the probability mass, we obtain
\begin{align*}
    P_T\bigl(\cB_r(x)\cap\St\bigr)
    &=
    P_0\!\left(
        \varphi(T,\cdot)^{-1}
        \bigl(\cB_r(x)\cap\St\bigr)
    \right)
    \\
    &\ge
    P_0\bigl(\cB_s(y)\cap\So\bigr)
    \\
    &\ge
    \rho\,2^{-n}
    \left(\frac{s}{R}\right)^n
    \\
    &=
    \rho\,2^{-n}e^{-nLT}
    \left(\frac{r}{R}\right)^n
    \\
    &=
    \Gamma_T(r).
\end{align*}
\end{proof}

\subsection{Sample Complexity}
\label{app:sample_complexity}

\textbf{Sample Complexity Upper Bound.}
We prove Theorem~\ref{thm:sample-upper-bound-inner} using a direct random-covering argument. The argument only requires that $\widehat S_N$ contain all endpoint samples. We first record the covering-number estimate used in the proof.

\begin{definition}[Covering number]
\label{def:covering-number}
For a bounded set $A\subset\mR^n$ and a radius $h>0$, the
$h$-covering number of $A$ is defined as
\begin{align*}
    \mathcal N(A,h)
    :=
    \min\left\{
        M\in\mathbb N:
        \exists z_1,\ldots,z_M\in A
        \text{ such that }
        A\subseteq
        \bigcup_{j=1}^M\cB_h(z_j)
    \right\}.
\end{align*}
\end{definition}

\begin{figure}[htbp]
    \centering
    \includegraphics[width=0.45\linewidth]{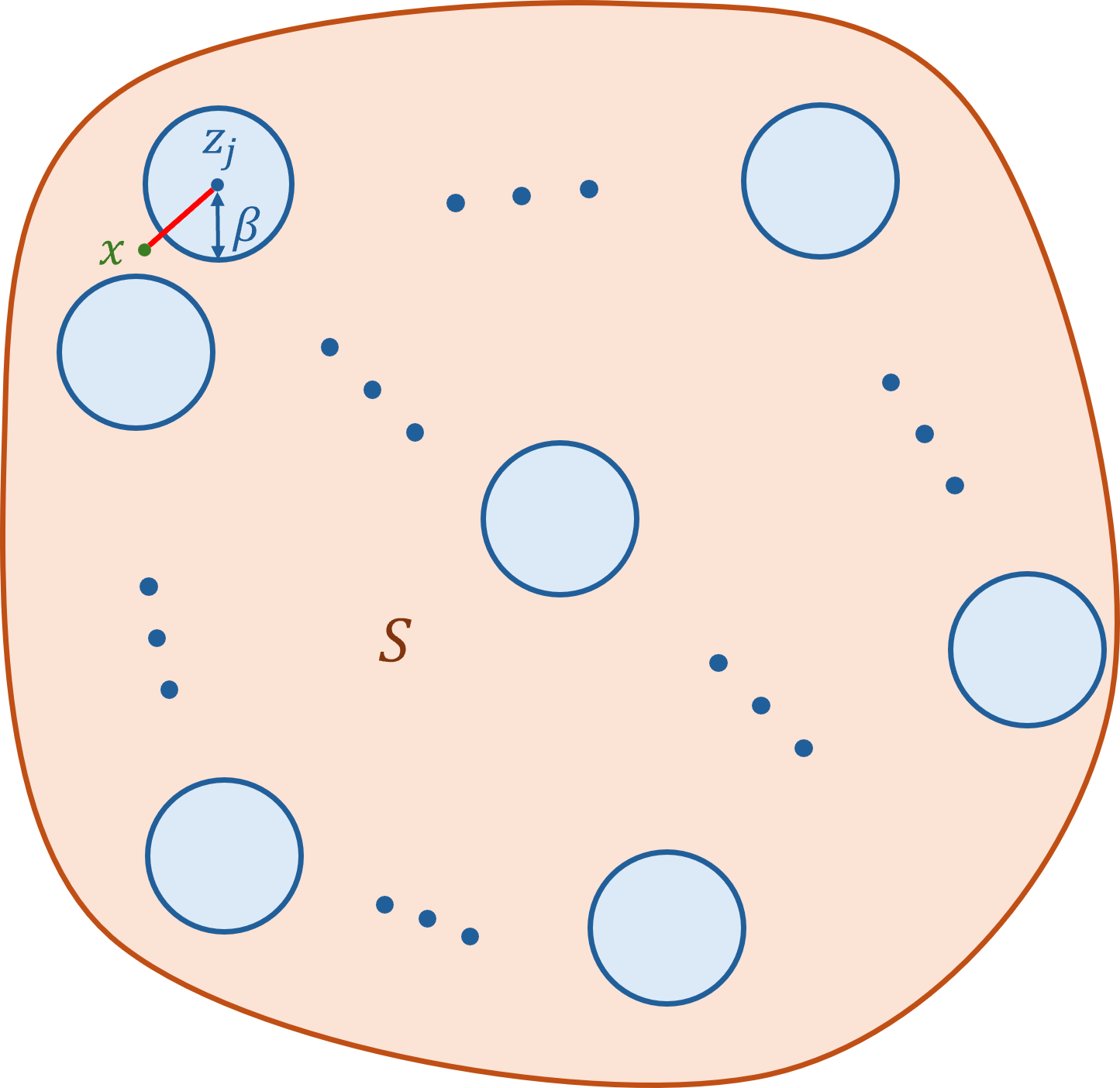}
    \caption{A finite covering of a bounded set.}
    \label{fig:covering}
\end{figure}

\begin{lemma}[Covering number of the reachable set]
\label{lem:covering-number-reachable-set}
Let $(S_0,F,P_0)\in\mathcal F_{R,L,r_0,\rho}$. Then, for every
$h\in(0,e^{LT}r_0]$, 
\begin{align*}
\mathcal N(\St,h) \le \frac{2^{2n}e^{nLT}R^n}{h^n}.
\end{align*}
In particular, $\mathcal N(\St,r/2) \le \frac{2^{3n}e^{nLT}R^n}{r^n}$ whenever $r\in(0,2e^{LT}r_0]$.
\end{lemma}

\begin{proof}
Since the flow map $\varphi(T,\cdot)$ is $e^{LT}$-Lipschitz on $\So$ and
maps $\So$ onto $\St$, every $e^{-LT}h$-cover of $\So$ is mapped to an
$h$-cover of $\St$. Hence,
\begin{align*}
    \mathcal N(\St,h)
    \le
    \mathcal N(\So,e^{-LT}h).
\end{align*}

It remains to bound the covering number of $\So$. Fix
$\beta\in(0,r_0]$, and let
$\{z_1,\ldots,z_M\}\subset\So$ be a maximal $\beta$-separated set, meaning
that $\|z_i-z_j\|>\beta$ whenever $i\neq j$, and that no additional point
of $\So$ can be added while preserving this property.

By maximality, $\{z_1,\ldots,z_M\}$ is a $\beta$-net of $\So$. Indeed, if
there existed $x\in\So$ such that
\begin{align*}
    x\notin
    \bigcup_{j=1}^M\cB_\beta(z_j),
\end{align*}
then $\|x-z_j\|>\beta$ for every $j$, so
$\{z_1,\ldots,z_M,x\}$ would remain $\beta$-separated, contradicting
maximality. Therefore,
\begin{align*}
    \So
    \subseteq
    \bigcup_{j=1}^M\cB_\beta(z_j),
\end{align*}
and consequently $\mathcal N(\So,\beta)\le M$. We now upper bound $M$. Since the points $z_1,\ldots,z_M$ are $\beta$-separated, the balls $\cB_{\beta/2}(z_j)$ are pairwise disjoint up to their boundaries. Hence, the sets
\begin{align*}
    \cB_{\beta/2}(z_j)\cap\So,
    \qquad
    j=1,\ldots,M,
\end{align*}
are also pairwise disjoint. By Proposition~\ref{prop:lower-lebesgue-initial-set}, applied with radius
$\beta/2$, we have
\begin{align*}
    \left|
        \cB_{\beta/2}(z_j)\cap\So
    \right|
    \ge
    \omega_n2^{-n}
    \left(\frac{\beta}{2}\right)^n
    =
    \omega_n2^{-2n}\beta^n,
    \qquad
    j=1,\ldots,M.
\end{align*}
Moreover, the positive-reach assumption implies that $\partial S_0$ has zero $n$-dimensional Lebesgue measure, and hence
$|\So|=|S_0|=\omega_nR^n$. Since the preceding sets are pairwise disjoint subsets of $\So$,
\begin{align*}
    \omega_nR^n &= |\So| \\
    &\ge \sum_{j=1}^M
    \left|
        \cB_{\beta/2}(z_j)\cap\So
    \right|
    \\
    &\ge
    M\omega_n2^{-2n}\beta^n.
\end{align*}
It follows that
\begin{align*}
    M
    \le
    \frac{2^{2n}R^n}{\beta^n}.
\end{align*}
Combining this inequality with $\mathcal N(\So,\beta)\le M$ gives
\begin{align*}
    \mathcal N(\So,\beta)
    \le
    \frac{2^{2n}R^n}{\beta^n},
    \qquad
    \beta\in(0,r_0].
\end{align*}

Finally, take $\beta=e^{-LT}h$. Since $h\le e^{LT}r_0$, we have
$\beta\le r_0$, and therefore
\begin{align*}
    \mathcal N(\St,h)
    &\le
    \mathcal N(\So,e^{-LT}h)
    \\
    &\le
    \frac{2^{2n}R^n}{(e^{-LT}h)^n}
    \\
    &=
    \frac{2^{2n}e^{nLT}R^n}{h^n}.
\end{align*}
Setting $h=r/2$ yields
\begin{align*}
    \mathcal N(\St,r/2)
    \le
    \frac{2^{3n}e^{nLT}R^n}{r^n}.
\end{align*}
\end{proof}

We now prove the sampling upper bound. The proof combines the endpoint local probability lower bound from Corollary~\ref{cor:endpoint-local-probability-lower-bound} with a covering argument. Specifically, we cover $\St$ by an $r/2$-net and show that, with high probability, every net ball contains at least one endpoint sample.

\noindent\textbf{Theorem~\ref{thm:sample-upper-bound-inner}}
(Sampling upper bound for inner coverage).
Let $(S_0,F,P_0)\in\mathcal F_{R,L,r_0,\rho}$, and let
$Y_1,\ldots,Y_N\overset{\iid}{\sim}P_T$. Suppose that the estimator
$\widehat S_N$ contains all endpoint samples, i.e.,
$Y_i\in\widehat S_N$ for every $i=1,\ldots,N$. Fix
$r\in(0,2e^{LT}r_0]$ and $\delta\in(0,1)$. If
\begin{align*}
    N
    \ge
    \frac{2^{2n}e^{nLT}R^n}{\rho r^n}
    \left[
        \log\left(
            \frac{2^{3n}e^{nLT}R^n}{r^n}
        \right)
        +
        \log\frac1\delta
    \right],
\end{align*}
then, with probability at least $1-\delta$,
\begin{align*}
    \sup_{x\in\St}d(x,\widehat S_N)
    \le
    r.
\end{align*}

\begin{proof}[Proof of Theorem~\ref{thm:sample-upper-bound-inner}]
By Corollary~\ref{cor:endpoint-local-probability-lower-bound}, $\Gamma_T(s) = \rho\,2^{-n}e^{-nLT} \left(\frac{s}{R}\right)^n$. Therefore, define
\begin{align*}
    q_r
    &:=
    \Gamma_T(r/2)
    \\
    &=
    \rho\,2^{-n}e^{-nLT}
    \left(\frac{r}{2R}\right)^n
    \\
    &=
    \rho\,2^{-2n}e^{-nLT}
    \left(\frac rR\right)^n.
\end{align*}

Let $\{z_1,\ldots,z_M\}\subset\St$ be an $r/2$-net of $\St$. By
Lemma~\ref{lem:covering-number-reachable-set}, we may choose this net so that $M \le \frac{2^{3n}e^{nLT}R^n}{r^n}$. By Corollary~\ref{cor:endpoint-local-probability-lower-bound}, for every
$j=1,\ldots,M$,
\begin{align*}
    P_T(\cB_{r/2}(z_j)\cap\St)
    \ge
    \Gamma_T(r/2)
    =
    q_r.
\end{align*}
Since $P_T$ is supported on $\St$, this is equivalent to
\begin{align*}
    P_T(\cB_{r/2}(z_j))
    \ge
    q_r.
\end{align*}

For each $j=1,\ldots,M$, define the bad event
\begin{align*}
    E_j
    :=
    \left\{
        \cB_{r/2}(z_j)
        \cap
        \{Y_1,\ldots,Y_N\}
        =
        \varnothing
    \right\}.
\end{align*}
Because $Y_1,\ldots,Y_N$ are independent samples from $P_T$,
\begin{align*}
    P_T^N(E_j)
    &=
    \left(
        1-P_T(\cB_{r/2}(z_j))
    \right)^N
    \\
    &\le
    (1-q_r)^N
    \\
    &\le
    e^{-Nq_r}.
\end{align*}
Therefore, by the union bound,
\begin{align*}
    P_T^N\left(
        \bigcup_{j=1}^M E_j
    \right)
    &\le
    \sum_{j=1}^M P_T^N(E_j)
    \\
    &\le
    Me^{-Nq_r}.
\end{align*}

The definition of $q_r$ gives the exact identity
\begin{align*}
    \frac1{q_r}
    =
    \frac{2^{2n}e^{nLT}R^n}{\rho r^n}.
\end{align*}
Thus, the assumed sample-size condition implies
\begin{align*}
    Nq_r
    &\ge
    \log\left(
        \frac{2^{3n}e^{nLT}R^n}{r^n}
    \right)
    +
    \log\frac1\delta
    \\
    &\ge
    \log M+\log\frac1\delta.
\end{align*}
Consequently,
\begin{align*}
    Me^{-Nq_r}
    \le
    \delta.
\end{align*}
Hence, with probability at least $1-\delta$, none of the events
$E_1,\ldots,E_M$ occurs. Equivalently, every ball centered at a point of
the $r/2$-net contains at least one endpoint sample.

On this event, fix any $x\in\St$. Since
$\{z_1,\ldots,z_M\}$ is an $r/2$-net of $\St$, there exists
$j\in\{1,\ldots,M\}$ such that
\begin{align*}
    \|x-z_j\|
    \le
    \frac r2.
\end{align*}
Since $E_j$ does not occur, there exists an endpoint sample $Y_i$ such that
\begin{align*}
    \|Y_i-z_j\|
    \le
    \frac r2.
\end{align*}
By the sample-containment assumption, $Y_i\in\widehat S_N$. Therefore,
\begin{align*}
    d(x,\widehat S_N) \le \|x-Y_i\| \le \|x-z_j\|+\|z_j-Y_i\| \le r.
\end{align*}
Since this holds for every $x\in\St$, we conclude that
\begin{align*}
    \sup_{x\in\St}d(x,\widehat S_N) \le r.
\end{align*}
\end{proof}

The preceding theorem controls only the inner directed error from the reachable set to the estimator. We now add the estimator-dependent outer deviation to obtain a full Hausdorff bound.

\noindent\textbf{Corollary~\ref{cor:hausdorff-bound-outer-deviation}} (Hausdorff bound with outer deviation). Under the conditions of Theorem~\ref{thm:sample-upper-bound-inner}, suppose in addition that the estimator satisfies the outer-deviation condition $\widehat S_N\subseteq \cB_{\eta}(S_T)$ for some $\eta\ge0$. Then, with probability at least $1-\delta$,
\begin{align*}
    d_H(S_T,\widehat S_N) \le \max\{r,\eta\}.
\end{align*}

\begin{proof}[Proof of Corollary~\ref{cor:hausdorff-bound-outer-deviation}]
By Theorem~\ref{thm:sample-upper-bound-inner}, with probability at least
$1-\delta$, $\sup_{x\in\St} d(x,\widehat S_N)\le r$. The outer-deviation condition $\widehat S_N\subseteq \cB_\eta(S_T)$ implies $\sup_{z\in\widehat S_N} d(z,S_T)\le \eta$. Therefore, on the same event,
\begin{align*}
    d_H(S_T,\widehat S_N) & = \max \{ \sup_{x\in S_T}d(x,\widehat S_N), \sup_{z\in\widehat S_N}d(z,S_T)\} \\
    &\le \max\{r,\eta\}.
\end{align*}
Thus, with probability at least $1-\delta$,
\begin{align*}
    d_H(S_T,\widehat S_N)\le \max\{r,\eta\}.
\end{align*}
\end{proof}

\textbf{Sample Complexity Lower Bound.} For the lower bound, we construct two admissible punctured-ball domains whose endpoint supports are separated by more than $2r$, while their uniform sampling distributions have a large common component. This reduces Hausdorff support estimation to binary hypothesis testing. To separate the generic testing argument from the subsequent two-point construction, we first formulate the reduction for an arbitrary pair of candidate endpoint supports.

\begin{definition}[Hypothesis test induced by a support estimator]
\label{def:induced-binary-test}
Let $S^{(0)},S^{(1)}\subseteq\mathbb R^n$ be two nonempty compact sets, and let $\mathcal K(\mathbb R^n)$ denote the space of nonempty compact subsets of $\mathbb R^n$, equipped with the Hausdorff metric. Given a measurable support estimator $\widehat S_N:(\mathbb R^n)^N\to\mathcal K(\mathbb R^n),$ define the induced binary test $\psi=\psi(Y_1,\ldots,Y_N)\in\{0,1\}$ by
\begin{align*}
\psi(Y_1,\ldots,Y_N)
:=
\begin{cases}
0,
& d_H(\widehat S_N,S^{(0)})
  \le d_H(\widehat S_N,S^{(1)}),\\
1,
& d_H(\widehat S_N,S^{(1)})
  < d_H(\widehat S_N,S^{(0)}).
\end{cases}
\end{align*}
\end{definition}

The test selects the candidate support that is closer to the estimated set. If the estimator is accurate and the two candidate supports are separated by more than twice the target accuracy, then the induced test must be correct.

\begin{lemma}[Accurate support estimation induces a correct test]
\label{lem:estimator-induces-test}
Suppose $d_H(S^{(0)},S^{(1)})>2r$, and let $\psi$ be the test in Definition~\ref{def:induced-binary-test}. Then $\{\psi=1\} \subseteq \{d_H(\widehat S_N,S^{(0)})>r\}$, and $\{\psi=0\} \subseteq \{d_H(\widehat S_N,S^{(1)})>r\}$.
\end{lemma}

\begin{proof}
Suppose that $d_H(\widehat S_N,S^{(0)})\le r$. By the triangle inequality for the Hausdorff distance,
\begin{align*}
    d_H(\widehat S_N,S^{(1)}) &\ge d_H(S^{(0)},S^{(1)}) - d_H(\widehat S_N,S^{(0)})
    \\
    &> 2r-r = r.
\end{align*}
Consequently, $d_H(\widehat S_N,S^{(0)}) < d_H(\widehat S_N,S^{(1)})$, and hence $\psi=0$. Taking the contrapositive gives $\{\psi=1\} \subseteq \{d_H(\widehat S_N,S^{(0)})>r\}$. The second inclusion follows by exchanging the roles of $S^{(0)}$ and $S^{(1)}$.
\end{proof}

We next record an elementary consequence of the definition of reach that will be used to verify the admissibility of the two-point construction.

\begin{lemma}[Reach of a separated union]
\label{lem:reach-separated-union}
Let $A,B\subseteq\mathbb R^n$ be nonempty closed sets. Then
\begin{align*}
    \reach(A\cup B)
    \ge
    \min\left\{
        \reach(A),
        \reach(B),
        \frac12\dist(A,B)
    \right\},
\end{align*}
where $\dist(A,B) := \inf\{\|a-b\|:a\in A,\ b\in B\}$.
In particular, if $\reach(A)\ge r_0, \reach(B)\ge r_0, \dist(A,B)\ge 2r_0$, then $\reach(A\cup B)\ge r_0$.
\end{lemma}

\begin{proof}
Set $s := \min\{ \reach(A), \reach(B), \frac12\dist(A,B)\}$. If $s=0$, the result is immediate. Therefore, suppose that $s>0$. Fix $x\in\mathbb R^n$ such that $d(x,A\cup B)<s$. Since $d(x,A\cup B) = \min\{d(x,A),d(x,B)\}$, at least one of $d(x,A)$ and $d(x,B)$ is strictly smaller than $s$. They cannot both be strictly smaller than $s$. Indeed, if $d(x,A)<s$ and $d(x,B)<s$, then
\begin{align*}
    \dist(A,B) \le d(x,A)+d(x,B) < 2s \le \dist(A,B),
\end{align*}
which is a contradiction. Thus exactly one of the two distances is smaller than $s$. Without loss of generality, suppose that $d(x,A)<s, d(x,B)\ge s$. Since $d(x,A)<s\le\reach(A)$, the projection of $x$ onto $A$ is unique. Let $a_x$ denote this unique projection. Moreover,
\begin{align*}
    \|x-a_x\|
    =
    d(x,A)
    <
    s
    \le
    d(x,B),
\end{align*}
so no point of $B$ can be a nearest point of $x$ in $A\cup B$. Therefore,
\begin{align*}
    \Pi_{A\cup B}(x)
    =
    \Pi_A(x)
    =
    \{a_x\}.
\end{align*}
The case $d(x,B)<s$ is analogous. Hence every point at distance strictly less
than $s$ from $A\cup B$ has a unique nearest point in $A\cup B$. By the
definition of reach,
\begin{align*}
    \reach(A\cup B)\ge s.
\end{align*}
\end{proof}

The following standard testing inequality relates binary testing error to the overlap of two probability measures.

\begin{lemma}[Testing lower bound via overlap]
\label{lem:testing-total-variation}
Let $P$ and $Q$ be probability measures on a measurable space $(\Omega,\mathcal F)$. Set $\mu:=P+Q, p:=\frac{dP}{d\mu}, q:=\frac{dQ}{d\mu}.$ Then, for any measurable binary test
$\psi:\Omega\to\{0,1\}$,
\begin{align*}
    P(\psi=1)+Q(\psi=0)
    \ge
    \int_\Omega \min\{p,q\}\,d\mu.
\end{align*}
\end{lemma}

\begin{proof}
Let $A_0:=\{\psi=0\}$ and $A_1:=\{\psi=1\}$. Then
\begin{align*}
    P(\psi=1)+Q(\psi=0)
    &= \int_{A_1}p\,d\mu + \int_{A_0}q\,d\mu\\
    &= \int_\Omega [ p\mathbf 1_{A_1} + q\mathbf 1_{A_0}]d\mu.
\end{align*}
At every point,
$p\mathbf 1_{A_1}+q\mathbf 1_{A_0}\ge\min\{p,q\}$.
Therefore,
\begin{align*}
    P(\psi=1)+Q(\psi=0) \ge \int_\Omega\min\{p,q\}\,d\mu.
\end{align*}
\end{proof}

We now prove the minimax lower bound by constructing two admissible instances
whose endpoint supports are well separated in Hausdorff distance, while their
endpoint sampling distributions have a large common component.

\noindent\textbf{Theorem~\ref{thm:minimax-hausdorff-lower-bound}}
(Minimax Hausdorff lower bound).
Suppose that $R\ge 4r_0, \rho\in(0,1]$, and let $0<r \le 2^{-\frac{n+1}{n}}e^{LT}r_0, \delta\in\left(0,\frac12\right)$. If
\begin{align*}
    N
    <
    \frac{e^{nLT}R^n}{2^{n+1}r^n}
    \log\frac{1}{2\delta},
\end{align*}
then, for every estimator $\widehat S_N$, there exists an instance $ (S_0,F,P_0)\in\mathcal F_{R,L,r_0,\rho}$
such that
\begin{align*}
    P_T^N\left(
        d_H(\widehat S_N,\overline{S_T})>r
    \right) > \delta.
\end{align*}

\begin{proof}[Proof of Theorem~\ref{thm:minimax-hausdorff-lower-bound}]
Fix an arbitrary estimator $\widehat S_N$. We construct two admissible instances indexed by $\ell\in\{0,1\}$. Because $2e^{-LT}r \le 2^{-1/n}r_0 < r_0$, and because the assumed upper bound on $N$ is strict, continuity with respect to $\alpha$ allows us to choose a sufficiently small $\alpha>0$ such that, with $h:=(2+\alpha)e^{-LT}r$, we have $h<r_0$ and
\begin{align}
    N
    <
    \frac{e^{nLT}R^n}
    {2(2+\alpha)^n r^n}
    \log\frac{1}{2\delta}.
    \label{eq:alpha-sample-condition}
\end{align}

Define $\bar R := (R^n+h^n)^{1/n}$. Since $\bar R\ge R\ge4r_0$
and $h<r_0$, we have
\begin{align*}
    \bar R>2r_0+2h.
\end{align*}
Hence there exists $\varepsilon>0$ such that $2r_0+2h+\varepsilon<\bar R$. Let $e_1$ denote the first Euclidean basis vector and define $a_0 := -(h+\frac{\varepsilon}{2})e_1, a_1 := (h+\frac{\varepsilon}{2})e_1$.
Then
\begin{align*}
    \|a_0-a_1\|
    =
    2h+\varepsilon
    >
    2h,
\end{align*}
so the two closed balls $\overline{\mathcal B_h(a_0)}$ and $\overline{\mathcal B_h(a_1)}$ are disjoint. Moreover,
\begin{align*}
    \dist\!\left(
        \overline{\mathcal B_h(a_i)},
        \mathbb R^n\setminus\mathcal B_{\bar R}^{\circ}(0)
    \right)
    &=
    \bar R-\|a_i\|-h
    \\
    &=
    \bar R-2h-\frac{\varepsilon}{2}
    \\
    &>
    2r_0,
\end{align*}
for $i=0,1$. Define $S_0^{(0)} := \mathcal B_{\bar R}^{\circ}(0) \setminus \overline{\mathcal B_h(a_0)}$, and $S_0^{(1)} := \mathcal B_{\bar R}^{\circ}(0) \setminus \overline{\mathcal B_h(a_1)}$. Both sets are bounded and open. Furthermore,
\begin{align*}
    |S_0^{(i)}|
    &= \omega_n\bar R^n-\omega_nh^n \\
    &= \omega_nR^n,
    \qquad i=0,1.
\end{align*}

We next verify the reach condition. Write $A := \mathbb R^n\setminus\mathcal B_{\bar R}^{\circ}(0), B_i := \overline{\mathcal B_h(a_i)}$. Then $(S_0^{(i)})^c=A\cup B_i$.
The set $B_i$ is closed and convex, and therefore $\reach(B_i)=+\infty$. Moreover, $\reach(A)=\bar R$ and $\dist(A,B_i)>2r_0$. Since $\bar R\ge r_0$, Lemma~\ref{lem:reach-separated-union} gives
\begin{align*}
    \reach((S_0^{(i)})^c)
    &= \reach(A\cup B_i)\\
    &\ge \min\left\{ \bar R, +\infty, \frac12\dist(A,B_i)
    \right\}
    \\
    &\ge
    r_0.
\end{align*}

Let $P_0^{(i)} := \operatorname{Unif}(S_0^{(i)}), i=0,1.$ Its density is $p_0^{(i)}(x) = \frac{1}{\omega_nR^n} \mathbf 1_{S_0^{(i)}}(x)$. Since $\rho\in(0,1]$,
\begin{align*}
    p_0^{(i)}(x) \ge \frac{\rho}{\omega_nR^n} = \frac{\rho}{|S_0^{(i)}|}
\end{align*}
for almost every $x\in S_0^{(i)}$. Hence both sampling laws satisfy the density condition in Definition~\ref{def:problem_family}.

For both instances, use the vector field $F(x)=Lx$, whose flow is $\varphi(T,x)=e^{LT}x$. This vector field is $L$-Lipschitz. Define the endpoint supports $\overline{S_T^{(i)}} := \varphi(T,\overline{S_0^{(i)}}), i=0,1$. The flow maps each initial hole of radius $h$ to a hole of radius $e^{LT}h = (2+\alpha)r$.
Let $y_0:=e^{LT}a_0$. Because the two initial holes are disjoint, $a_0\in S_0^{(1)}$. Hence $y_0\in\overline{S_T^{(1)}}$. On the other hand, $d(a_0,\overline{S_0^{(0)}})=h$, and therefore $d(y_0,\overline{S_T^{(0)}}) = e^{LT}h = (2+\alpha)r$.
It follows that
\begin{align*}
    d_H(\overline{S_T^{(0)}},\overline{S_T^{(1)}})
    & \ge d(y_0, \overline{S_T^{(0)}}) \\
    & =(2+\alpha)r \\
    & >2r.
\end{align*}

It remains to compare the two endpoint sampling distributions. Since the two
holes are disjoint,
\begin{align*}
    |S_0^{(0)}\cap S_0^{(1)}|
    &=
    \omega_n\bar R^n
    -
    2\omega_nh^n
    \\
    &=
    \omega_nR^n-\omega_nh^n.
\end{align*}
Consequently,
\begin{align*}
    \frac{|S_0^{(0)}\cap S_0^{(1)}|}
    {|S_0^{(i)}|}
    =
    1-\left(\frac hR\right)^n.
\end{align*}

Let $A_i:=\varphi(T,S_0^{(i)})$. Because $\varphi(T,x)=e^{LT}x$, the endpoint distribution $P_T^{(i)}$ is uniform on $A_i$, and $|A_i| = e^{nLT}\omega_nR^n$. Moreover, $|A_0\cap A_1| = e^{nLT}|S_0^{(0)}\cap S_0^{(1)}|$. Thus the one-sample overlap coefficient is
\begin{align*}
    \int
    \min\{dP_T^{(0)},dP_T^{(1)}\}
    & = \frac{|A_0\cap A_1|}{|A_i|} \\
    & = 1-\left(\frac hR\right)^n.
\end{align*}
For $N$ independent endpoint samples, the product densities are
\begin{align*}
    p_{T,N}^{(i)}(y_1,\ldots,y_N)
    =
    |A_i|^{-N}
    \mathbf 1_{A_i^N}(y_1,\ldots,y_N).
\end{align*}
Since the two product densities have the same constant value on their
respective supports,
\begin{align*}
    \int
    \min\{
        d(P_T^{(0)})^N,
        d(P_T^{(1)})^N
    \}
    &=
    \frac{|(A_0\cap A_1)^N|}{|A_i|^N}
    \\
    &=
    \left[
        1-\left(\frac hR\right)^n
    \right]^N.
\end{align*}
Let $\psi$ be the test induced by $\widehat S_N$ according to
Definition~\ref{def:induced-binary-test}. Lemma~\ref{lem:testing-total-variation}
gives
\begin{align*}
    (P_T^{(0)})^N(\psi=1)
    +
    (P_T^{(1)})^N(\psi=0)
    &\ge
    \int
    \min\{
        d(P_T^{(0)})^N,
        d(P_T^{(1)})^N
    \}
    \\
    &=
    \left[
        1-\left(\frac hR\right)^n
    \right]^N.
\end{align*}
Since $d_H(\overline{S_T^{(0)}}, \overline{S_T^{(1)}}) > 2r$,
Lemma~\ref{lem:estimator-induces-test} implies
\begin{align*}
    \{\psi=1\}
    \subseteq
    \left\{
        d_H(
            \widehat S_N,
            \overline{S_T^{(0)}}
        )>r
    \right\},
\end{align*}
and
\begin{align*}
    \{\psi=0\}
    \subseteq
    \left\{
        d_H(
            \widehat S_N,
            \overline{S_T^{(1)}}
        )>r
    \right\}.
\end{align*}
Therefore,
\begin{align*}
    &
    \sup_{\ell\in\{0,1\}}
    (P_T^{(\ell)})^N
    \left(
        d_H(
            \widehat S_N,
            \overline{S_T^{(\ell)}}
        )>r
    \right)
    \\
    &\qquad\ge
    \frac12
    \left[
        (P_T^{(0)})^N(\psi=1)
        +
        (P_T^{(1)})^N(\psi=0)
    \right]
    \\
    &\qquad\ge
    \frac12
    \left[
        1-\left(\frac hR\right)^n
    \right]^N.
\end{align*}

Since $h<r_0\le\frac R4$, we have $\left(\frac hR\right)^n < 4^{-n} \le \frac12$. Using $\log(1-u)\ge-2u, 0\le u\le\frac12$, we obtain
\begin{align*}
    \left[
        1-\left(\frac hR\right)^n
    \right]^N
    &=
    \exp\left(
        N\log\left[
            1-\left(\frac hR\right)^n
        \right]
    \right)
    \\
    &\ge
    \exp\left(
        -2N\left(\frac hR\right)^n
    \right).
\end{align*}
Hence
\begin{align*}
    &
    \sup_{\ell\in\{0,1\}}
    (P_T^{(\ell)})^N
    \left(
        d_H(
            \widehat S_N,
            \overline{S_T^{(\ell)}}
        )>r
    \right)
    \\
    &\qquad\ge
    \frac12
    \exp\left(
        -2N\left(\frac hR\right)^n
    \right)
    \\
    &\qquad=
    \frac12
    \exp\left(
        -2N(2+\alpha)^n e^{-nLT}
        \left(\frac rR\right)^n
    \right).
\end{align*}

By~\eqref{eq:alpha-sample-condition}, $2N(2+\alpha)^n e^{-nLT} \left(\frac rR\right)^n < \log\frac{1}{2\delta}$. Therefore,
\begin{align*}
    &
    \sup_{\ell\in\{0,1\}}
    (P_T^{(\ell)})^N
    \left(
        d_H(
            \widehat S_N,
            \overline{S_T^{(\ell)}}
        )>r
    \right)
    \\
    &\qquad>
    \frac12
    \exp\left(
        -\log\frac{1}{2\delta}
    \right)
    =
    \delta.
\end{align*}

Thus at least one of the two admissible instances satisfies
\begin{align*}
    P_T^N\left(
        d_H(\widehat S_N,\overline{S_T})>r
    \right)
    >
    \delta.
\end{align*}
This proves the result.
\end{proof}

\begin{remark}[Worst-case and instance-dependent time dependence]
\label{rem:instance-dependent-time}
The factor $e^{nLT}$ in Theorem~\ref{thm:minimax-hausdorff-lower-bound} is a worst-case minimax dependence. The theorem gives a uniform impossibility result over the entire problem class; it does not imply that every fixed instance requires sample
complexity scaling as $e^{nLT}r^{-n}$. For a fixed instance, a natural quantity describing the local sampling bottleneck is
\begin{align*}
\gamma_T(r):=\inf_{y\in\overline{S_T}}P_T(\cB_r(y)),
\end{align*}
the smallest endpoint probability mass of an $r$-neighborhood centered on the endpoint support. Observing the least likely such neighborhood with constant probability generally requires on the order of $1/\gamma_T(r)$ samples, while uniform coverage of the entire support may introduce additional logarithmic factors
depending on its covering number and the desired confidence level.

The construction used in the proof shows that the worst-case exponential dependence can occur. Under the isotropic expanding flow $\varphi(T,x)=e^{LT}x$, an endpoint feature at spatial scale $r$ has a preimage at scale $e^{-LT}r$. Since the initial distribution is uniform on a set of volume $\omega_nR^n$, the probability mass distinguishing the two constructed instances is of order $(e^{-LT}r/R)^n=e^{-nLT}(r/R)^n$. Consequently, distinguishing their endpoint supports requires on the order of $e^{nLT}(R/r)^n$ samples, up to constants and confidence-dependent factors.

For sufficiently smooth nonlinear dynamics, the instance-dependent time dependence can be related to the local volume expansion of the flow. Let $J_T(x):=|\det D_x\varphi(T,x)|$. At points where the relevant derivatives and densities are defined, the change-of-variables formula gives $p_T(\varphi(T,x))=p_0(x)/J_T(x)$. Thus, for sufficiently small $r$ and an
interior endpoint $y=\varphi(T,x)$, $P_T(\cB_r(y))\approx\omega_n r^n p_0(x)/J_T(x)$, with density variation, boundary geometry, and finite-scale distortion affecting the constants.

Define the maximal finite-time log-volume expansion by $\Lambda_T:=\sup_{x\in\overline{S_0}}\log J_T(x)$. Under appropriate
two-sided density bounds and uniform local geometric regularity, regions experiencing the largest volume expansion can have $r$-neighborhood mass on the order of $e^{-\Lambda_T}(r/R)^n$, leading to an instance-dependent sampling burden on the order of $e^{\Lambda_T}(R/r)^n$, up to logarithmic coverage and confidence factors. The global $L$-Lipschitz assumption gives
$\|D_x\varphi(T,x)\|_{\mathrm{op}}\le e^{LT}$ and $J_T(x)\le e^{nLT}$, and hence $\Lambda_T\le nLT$. For the isotropic flow
$\varphi(T,x)=e^{\lambda T}x$, one has $\Lambda_T=n\lambda T$, recovering the dependence $e^{n\lambda T}(R/r)^n$. Thus, the minimax theorem captures the largest volume expansion permitted by the global Lipschitz assumption, whereas a fixed instance may exhibit a substantially smaller effective expansion. Nevertheless, persistent positive volume expansion can still
create an exponential-in-time sampling burden for Hausdorff-accurate reachable-set recovery.
\end{remark}

\section{Experimental Details}
\label{app:additional_experiments}

This section provides additional implementation details and experimental results. We first describe the simulation platform, hardware configuration,
robot-arm dynamics, parameter settings, and the estimator based on Christoffel functions used in the supplementary experiments. We then report additional
results on dimension-dependent approximation of robot-arm uncertainty propagation, followed by reachable-set approximation results obtained using the Christoffel estimator.

\subsection{Experimental Setup}
\label{app:experimental_setup_additional_results}

\paragraph{Simulation Platform and Hardware.}
All simulations are run on a 3.2 GHz AMD Ryzen 7 7735HS CPU with 16 GB RAM. The robot-arm simulations are implemented in MuJoCo 3.8.1. We use deterministic MuJoCo MJCF models with a fixed integration time step of $2\times 10^{-3}$ seconds. The approximation error is computed offline from simulated terminal samples. In the main experiments in Section~\ref{sec:exp}, we use the convex hull of the sampled endpoints as the reachable-set estimator.

\paragraph{Dynamics and Parameter Settings of Robotic Arms.}
We consider vertical planar serial $n$-link robot arms with $n\in\{2,3,4\}$. The state is $x=[q^\top,v^\top]^\top\in\mathbb{R}^{2n}$, where $q\in\mathbb{R}^n$ denotes joint angles and $v=\dot q\in\mathbb{R}^n$ denotes joint velocities. Each link has length $0.5$, capsule radius $0.035$, density $1000$, joint damping $0.2$, and armature $0.01$. Gravity is enabled with acceleration $(0,0,-9.81)$. MuJoCo simulates the rigid-body dynamics
\begin{align*}
M(q)\dot v + C(q,v)v + g(q) = \tau ,
\end{align*}
where $\tau\in\mathbb{R}^n$ is the vector of joint torques. For the simulation, we consider a non-adaptive inverse-dynamics tracking controller. The desired reference trajectory is $q_{d,i}(t) = q_{c,i} + A_i \sin(\omega t+\phi_i),$ with $q_c = \mathrm{linspace}(0.25,0.65,n),\, A_i=0.08,\, \omega=0.5,\, \phi=\mathrm{linspace}(0,\pi/3,n)$. The analytic derivatives $\dot q_d(t)$ and $\ddot q_d(t)$ are used in the controller. Let $e=q-q_d(t)$ and $\dot e=v-\dot q_d(t)$. MuJoCo inverse dynamics is used to compute \begin{align*} \tau_{\rm track} = M(q)\ddot q_d(t)+C(q,v)v+g(q). \end{align*} The applied torque is \begin{align*} \tau = \tau_{\rm track}-K_p e-K_d\dot e . \end{align*} We use weak tracking gains \begin{align*} K_p=0.01I_n,\qquad K_d=0.005I_n, \end{align*} and clip torques to $[-100,100]$ for numerical stability. The weak gains are chosen deliberately: the benchmark is designed for uncertainty propagation and finite-sample reachable-set approximation, rather than aggressive trajectory tracking or controller design.

Let $K:=[-\rho_q,\rho_q]^n\times[-\rho_v,\rho_v]^n\subset\mR^{2n}$ and
$\cB_{r}(0,r)=\{y\in\mR^{2n}:\|y\| \le r_0\}$. And the initial uncertainty set is the box $S_0= K \circ \cB(0, r)$ with $\rho_q=\rho_v=0.1, r=0.01$. Each trajectory is propagated to time $T=1.0$ second, and the terminal propagated uncertainty set is approximated from sampled initial conditions. 

\paragraph{Estimator using Christoffel function.}

To further validate the geometric implications of our theory, we use the empirical inverse Christoffel function as an independent verification tool~\citep{devonport2021data}. Given endpoint samples $\{Y_i\}_{i=1}^N$, we construct
\begin{align*}
\widehat M_{m,\sigma_0} = \sigma_0^2 I+\frac{1}{N}\sum_{i=1}^N z_m(Y_i)z_m(Y_i)^\top,
\end{align*}
where $z_m$ is the vector of monomials of degree at most $m$, and define
\begin{align*}
C(y)=z_m(y)^\top \widehat M_{m,\sigma_0}^{-1}z_m(y).
\end{align*}
Following the polynomial Christoffel estimator in~\citep[Algorithm 3]{devonport2023data}, we use the sublevel set $\widehat S_N^{\rm Ch}
=
\{y:C(y)\le \eta\}$
as a data-driven certificate of support coverage. In our setting, this construction is not used as a competing reachable-set estimator, but rather as an auxiliary verification procedure: if the endpoint samples generated under
our sampling scheme are sufficiently informative, the resulting Christoffel sublevel set should concentrate around the reachable set and provide an independent check of the predicted finite-sample behavior.

\subsection{Adversarial Sampling Details}
\label{app:adversarial-sampling}
For completeness, we describe the adversarial sampling heuristic over $S_0$ used in
the experiments to improve empirical coverage~\citep{goodfellow2015explaining,dong2018boosting}. Given the current endpoint cloud $\mathcal Y^i$, define $c^i := |\mathcal Y^i|^{-1}\sum_{Y\in\mathcal Y^i}Y$ and $Q^i := [ \frac{1}{|\mathcal Y^i|-1}\sum_{Y\in\mathcal Y^i}(Y-c^i)(Y-c^i)^\top +\lambda I]^{-1}$, where $\lambda>0$ is a regularization parameter. We then maximize the novelty objective $\mathcal L^i(x):=\|\varphi(T,x)-c^i\|_{Q^i}^2$ over $x\in S_0$, which encourages new samples to generate endpoints far from the current cloud after covariance normalization~\citep{lew2021sampling}. The full procedure is summarized in Algorithm~\ref{alg:adv_reachable_set_estimation}.

\noindent
\begin{algorithm}[H]
\caption{Reachable Set Estimation via Adversarial Sampling}
\label{alg:adv_reachable_set_estimation}
\small
\begin{algorithmic}[1]
\State \textbf{Input:} Initial samples $\{X_j^0\}_{j=1}^M\subset S_0$
\State \textbf{Parameters:} Stepsize $\eta$, iterations $n_{\rm adv}$
\State \textbf{Output:} Estimate $\widehat S_N^{\rm adv}$
\State $Y_j^0\gets \varphi(T,X_j^0),\quad j=1,\ldots,M$
\State $\mathcal Y^0\gets \{Y_j^0\}_{j=1}^M$
\For{$i=0,\ldots,n_{\rm adv}-1$}
    \State Compute $c^i,Q^i$ from $\mathcal Y^i$
    \For{$j=1,\ldots,M$}
        \State $X_j^{i+1}
        \gets
        \mathrm{Proj}_{S_0}
        \bigl(X_j^i+\eta\nabla_x\mathcal L^i(X_j^i)\bigr)$
        \State $Y_j^{i+1}\gets \varphi(T,X_j^{i+1})$
    \EndFor
    \State $\mathcal Y^{i+1}\gets
    \mathcal Y^i\cup\{Y_j^{i+1}\}_{j=1}^M$
\EndFor
\State $N\gets M(n_{\rm adv}+1)$
\State \Return $\widehat S_N^{\rm adv}=\mathrm C(\mathcal Y^{n_{\rm adv}})$
\end{algorithmic}
\end{algorithm}
Thus, the procedure allocates samples toward geometrically underexplored regions of the reachable set. At each iteration, the current endpoint cloud is updated along $\nabla_x\mathcal L^i$, projected back to $S_0$, and propagated through the dynamics. The newly generated endpoints are
then added to the sample cloud, and the final estimator is applied to all accumulated endpoints. For more details, please refer to~\citep{lew2021sampling}.

\subsection{Additional Results}
\label{app:additional_results}

\paragraph{Role of the uniform density lower bound.}
\label{app:role-density-lower-bound}
The condition $\supp(P_0)=\overline{S_0}$ guarantees that every nonempty open subset of $S_0$ has positive probability, but it provides no uniform quantitative lower bound on its mass. Consequently, regions where the density becomes small may remain poorly sampled at finite sample sizes. The assumption $p_0(x)\ge\rho/(\omega_nR^n)$ excludes this degeneracy, since
\begin{align*}
    P_0(A) \ge \rho\frac{|A|}{|S_0|}
\end{align*}
for every measurable $A\subseteq S_0$. Together with the local thickness of $S_0$ and the flow-distortion bounds, this yields the uniform local probability estimate used in Theorem~\ref{thm:sample-upper-bound-inner}.

To illustrate this role, consider the globally Lipschitz linear system $\dot x=2x,\ \dot y=0$ with $T=0.5$. The initial set is the unit disk $S_0=\{(x,y)\in\mathbb R^2:r(x,y)<1\}$ centered at $(1,0)$, where $r(x,y):=\sqrt{(x-1)^2+y^2}$. For $\beta\in\{0,2,4\}$, initial states are sampled from
\begin{align*}
    p_\beta(x,y)
    &=
    \frac{(1-r(x,y))^\beta}{Z_\beta}
    \mathbf 1_{\{r(x,y)<1\}},
    &
    Z_\beta
    &=
    \frac{2\pi}{(\beta+1)(\beta+2)}.
\end{align*}
All three distributions satisfy $\supp(P_\beta)=\overline{S_0}$. The case $\beta=0$ is uniform and satisfies the density lower bound with $\rho=1$. For $\beta=2$ and $\beta=4$, however, $p_\beta(x,y)\to0$ as $r(x,y)\uparrow1$, so $\inf_{(x,y)\in S_0}p_\beta(x,y)=0$. Thus, the three instances share the same support, geometry, and dynamics, while differing only in their boundary
sampling behavior.

Near the boundary, $1-r(x,y)$ equals the inward distance to $\partial S_0$. Consequently, for a boundary point $z$ and sufficiently small $\varepsilon$, the local probability mass scales as $P_\beta(\cB_\varepsilon(z))\asymp\varepsilon^{\beta+2}$. The linear flow preserves this local exponent up to constants. Hence, the local sampling burden becomes progressively more severe as $\beta$ increases, even though all three distributions have the same support.

The endpoint support at time $T$ is
\begin{align*}
    \overline{S_T} = \left\{
        (X,Y):
        \frac{(X-e^{2T})^2}{e^{4T}}+Y^2\le1
    \right\}.
\end{align*}
For $N\in\{10,30,100,300,10^3,3\times10^3,10^4,3\times10^4,
10^5,3\times10^5,10^6\}$, we propagate the samples to time $T$ and approximate $\overline{S_T}$ using convex-hull, fixed-radius union-of-balls, and Christoffel estimators. Figure~\ref{fig:density-lower-bound-ablation}
reports the mean approximate symmetric Hausdorff error over $50$ independent trials, with shaded regions showing the $5$--$95$ percentile range.

\begin{figure}[htbp]
    \centering
    \includegraphics[width=\linewidth]{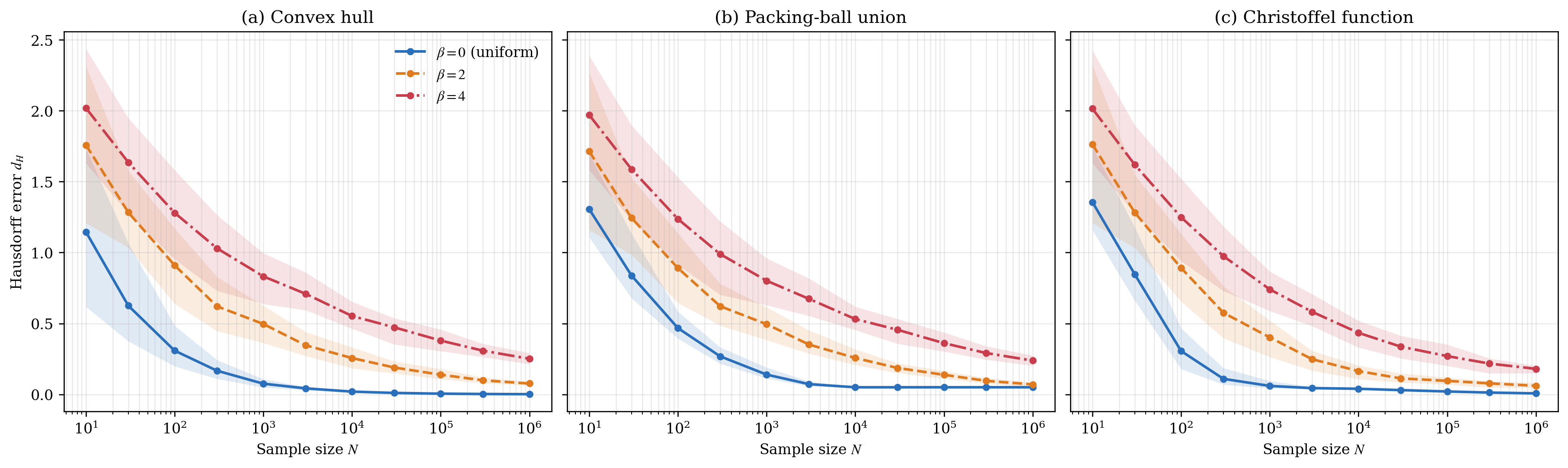}
    \caption{Reachable-set approximation under different boundary-density decay rates. Left to right: convex hull, fixed-radius union of balls ($h=0.05$), and Christoffel estimators. Curves show mean approximate symmetric Hausdorff error; shading indicates the $5$th--$95$th percentiles over $50$ trials.}
    \label{fig:density-lower-bound-ablation}
\end{figure}

The results show that full support alone does not yield a uniform finite-sample recovery rate. At $N=10^6$, the mean Christoffel errors are approximately $0.0068$, $0.0605$, and $0.1800$ for $\beta=0$, $2$, and $4$, respectively. For $\beta=4$, the error decreases from approximately $0.434$ at $N=10^4$ to $0.270$ at $N=10^5$ and $0.180$ at $N=10^6$. This slow error decay is consistent with the local-mass scaling $P_\beta(\cB_\varepsilon(z))\asymp\varepsilon^{\beta+2}$. Indeed, the
probability that a fixed boundary neighborhood receives no sample is approximately $\exp(-cN\varepsilon^{\beta+2})$. Thus, hitting one prescribed neighborhood with constant probability requires $N\asymp\varepsilon^{-(\beta+2)}$, whereas covering all $M(\varepsilon)$ neighborhoods in an $\varepsilon$-cover of the boundary
with probability at least $1-\delta$ generally requires
\begin{align*}
    N
    \gtrsim
    \varepsilon^{-(\beta+2)}
    \left(
        \log M(\varepsilon)+\log\frac{1}{\delta}
    \right).
\end{align*}
When $M(\varepsilon)$ grows polynomially in $1/\varepsilon$, this introduces
an additional factor of order $\log(1/\varepsilon)$.

The union-of-balls estimator exhibits an additional estimator-dependent effect. Under uniform sampling, its error approaches the fixed radius $h=0.05$ because a dense union of radius-$h$ balls extends approximately distance $h$ beyond the true support. This plateau is an outer-approximation effect rather than a sampling-coverage failure. For $\beta=2$ and $\beta=4$, the excess error above this radius is consistent with unresolved inner coverage caused by the small boundary mass. The corresponding separation between densities under the convex-hull and Christoffel estimators further shows that the finite-sample degradation is not specific to a single estimator, although its precise rate remains estimator-dependent. Overall, the experiment illustrates why full support alone is insufficient for a uniform quantitative guarantee and distinguishes the role of local sampling mass from estimator-dependent outer deviation. Consequently, a distribution-uniform finite-sample recovery rate requires a uniform quantitative lower bound on local probability mass, which can, for example, be ensured by a uniform lower bound on the sampling density.

\paragraph{Fitted Dimension-Dependent Slopes.} To summarize the dimension-dependent trend in Table~\ref{tab:robotarm-slopes}, we fit the magnitude of the empirical log--log slope as a function of the state dimension $n$ using 
\begin{align*} 
m(n)=\frac{1}{a n^b+c}, \qquad \text{slope}(n)\approx -m(n). 
\end{align*} 
Since the fitted curves are obtained from plots of $\log d_H$ versus $\log N$, they can also be interpreted as empirical sample-complexity exponents. Indeed, if \begin{align*} \log d_H \approx \alpha(n)-m(n)\log N,\end{align*} then $d_H \approx A_n N^{-m(n)}$, where $A_n = e^{\alpha(n)}$ is the dimension-dependent prefactor in the empirical scaling law; equivalently, it is the value predicted by the fitted power law at $N=1$. In practice, $A_n$ should be interpreted as an intercept parameter rather than as a reliable one-sample approximation error, since the fit is obtained over a finite range of sample sizes. This prefactor captures effects not represented by the slope alone, including the geometric scale of the reachable set, flow-induced expansion, sampling density, and estimator bias. Empirically, $A_n$ often increases with the state dimension, which further amplifies the sample requirement. This interpretation is consistent with classical support- and level-set estimation results, where Hausdorff-type recovery rates and covering complexities depend explicitly on the ambient or intrinsic dimension~\citep{rodriguez2022data}. Thus achieving accuracy $d_H\le r$ requires, up to dimension-dependent constants, $N \gtrsim \left(\frac{A_n}{r}\right)^{1/m(n)} = \left(\frac{A_n}{r}\right)^{a n^b+c}$. Therefore, the fitted denominator $a n^b+c$ represents the empirical exponent with which the sample budget must grow as the target accuracy $r$ decreases.

The corresponding parameters and fitted curves are shown in Table~\ref{tab:robotarm-slope-fit} and Figure~\ref{fig:robotarm-slope-fit}, respectively.

\begin{table}[htbp] 
\centering 
\caption{Fitted parameters for the empirical slope model $m(n)=1/(a n^b+c)$, where $n$ is the state dimension and $\text{slope}(n)\approx -m(n)$. The implied sample-complexity exponent is $1/m(n)=a n^b+c$.} 
\label{tab:robotarm-slope-fit}
\begin{tabular}{lccc} 
\toprule 
Sampling method & $a$ & $b$ & $c$ \\ 
\midrule Uniform sampling & 0.5049 & 1.0709 & 1.3353 \\ Adversarial sampling & 0.9488 & 0.7203 & 0.2535 \\ 
\bottomrule 
\end{tabular} 
\end{table}

\begin{figure}[htbp] \centering \includegraphics[width=0.7\linewidth]{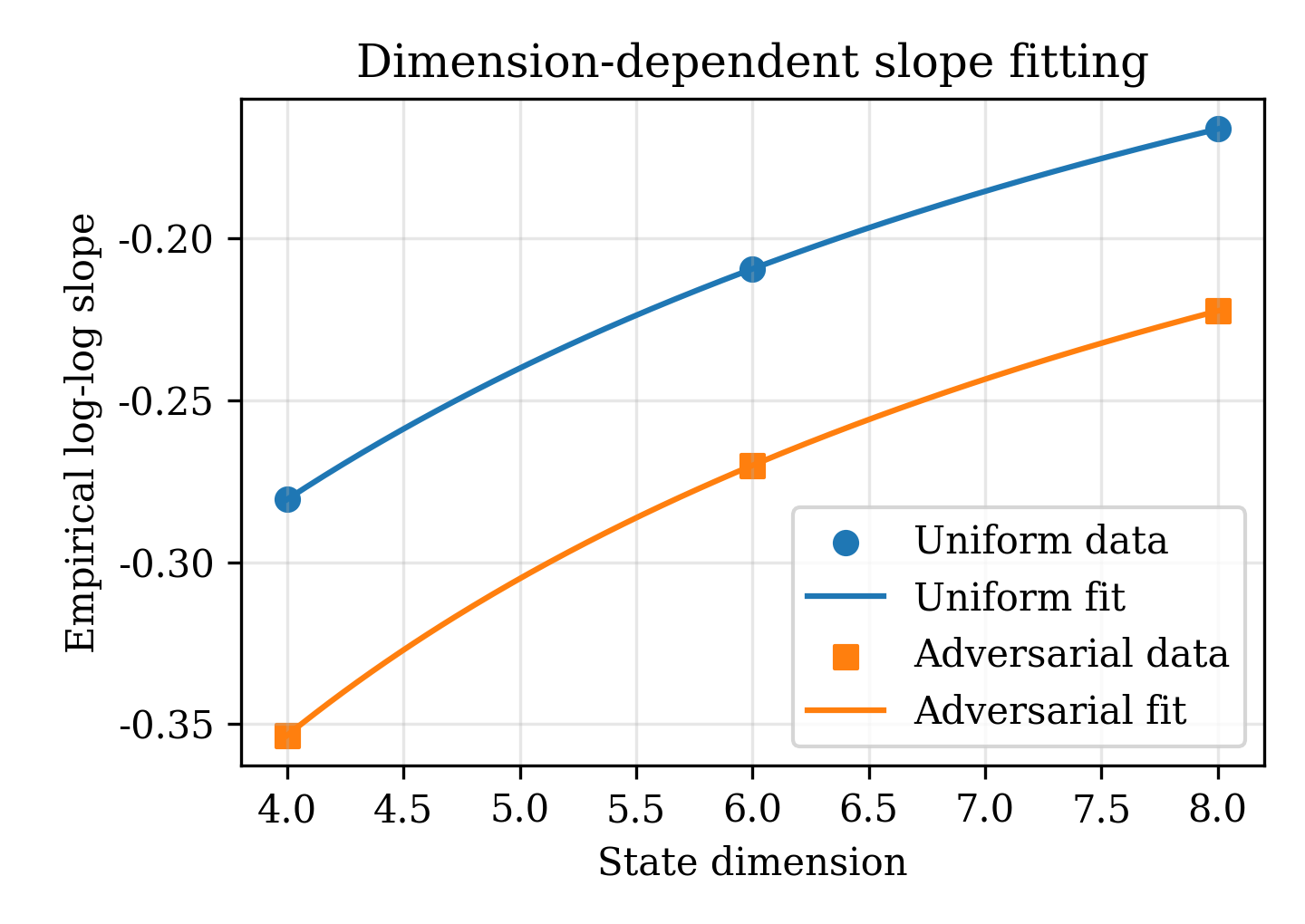} \caption{Fitted dimension-dependent empirical slopes for robot-arm uncertainty propagation. The fitted model captures the flattening of the log--log convergence slopes as the state dimension increases.} \label{fig:robotarm-slope-fit} \end{figure}

Substituting the fitted parameters in Table~\ref{tab:robotarm-slope-fit} into the relation $ N \gtrsim \left(\frac{A_n}{r}\right)^{a n^b+c} $ shows that the sample budget grows as \begin{align*} 
N_{\rm unif} \gtrsim \left(\frac{A_n}{r}\right)^{0.5049\,n^{1.0709}+1.3353}, N_{\rm adv} \gtrsim \left(\frac{A_n}{r}\right)^{0.9488\,n^{0.7203}+0.2535}. \end{align*} 
Equivalently, taking logarithms on both sides yields
\begin{align*} 
\log N_{\rm unif} &\gtrsim \bigl(0.5049\,n^{1.0709}+1.3353\bigr) \log\frac{A_n}{r}, \\
\log N_{\rm adv} &\gtrsim \bigl(0.9488\,n^{0.7203}+0.2535\bigr) \log\frac{A_n}{r}.
\end{align*} 
Thus the empirical sample requirement does not grow only logarithmically with dimension. Both uniform and adversarial sampling exhibit an exponential-type dependence on the state dimension through the exponent multiplying $\log(A_n/r)$. Over the tested dimensions, the fitted inverse-rate exponent increases under both sampling methods, indicating that achieving smaller Hausdorff error becomes progressively more sample-demanding as the state dimension grows. Adversarial sampling has a smaller empirical inverse-rate exponent at each tested dimension, indicating improved finite-sample efficiency, but its exponent also increases with dimension. Thus, the experiments provide evidence that dimension-dependent degradation persists under targeted sampling over the tested range. Because adversarial sampling is adaptive and produces statistically dependent samples, this conclusion is empirical and does not constitute a minimax lower bound for adaptive sampling.

\paragraph{Sample-Budget-Dependent Improvement of Adversarial Sampling.} Section~\ref{sec:exp} reports the main dimension-dependent robot-arm results. Here we provide an additional comparison showing how the benefit of adversarial sampling depends on the sample budget. Table~\ref{tab:robotarm-adv-improvement-combined} reports the absolute and relative improvement of adversarial sampling over uniform sampling for the $2$-, $3$-, and $4$-link robot arms. 

The advantage of adversarial sampling is not uniform across all sample sizes. When $N$ is very small, adversarial updates may over-concentrate samples along a few extreme directions and reduce global coverage. As the sample budget grows, the sampler has enough coverage to exploit informative directions, and the relative improvement becomes positive and more stable. At $N=3000$, the relative improvements are approximately $39.09\%$, $35.56\%$, and $30.58\%$ for the $2$-, $3$-, and $4$-link arms, respectively. 

\begin{table}[htbp] 
\centering 
\caption{Absolute (Abs.) and relative (Rel.) improvement of adversarial sampling over uniform sampling for robot-arm uncertainty propagation.} \label{tab:robotarm-adv-improvement-combined} \begin{tabular}{rcc cc cc}
\toprule & \multicolumn{2}{c}{$n=2$} & \multicolumn{2}{c}{$n=3$} & \multicolumn{2}{c}{$n=4$} \\ 
\cmidrule(lr){2-3} \cmidrule(lr){4-5} \cmidrule(lr){6-7} $N$ & Abs. & Rel. (\%) & Abs. & Rel. (\%) & Abs. & Rel. (\%) \\ 
\midrule 1 & -0.0117 & -3.35 & 0.0000 & 0.01 & -0.0098 & -2.15 \\ 
3 & 0.0236 & 8.27 & -0.0157 & -4.88 & -0.0126 & -3.47 \\ 
10 & 0.0356 & 17.79 & 0.0122 & 4.93 & 0.0011 & 0.37 \\ 
30 & 0.0365 & 25.86 & 0.0374 & 18.01 & 0.0250 & 10.26 \\ 
100 & 0.0397 & 36.72 & 0.0316 & 20.70 & 0.0386 & 18.84 \\ 
300 & 0.0317 & 40.10 & 0.0307 & 25.35 & 0.0462 & 25.90 \\ 
1000 & 0.0226 & 41.17 & 0.0290 & 30.12 & 0.0371 & 26.68 \\ 
3000 & 0.0144 & 39.09 & 0.0280 & 35.56 & 0.0355 & 30.58 \\ 
\bottomrule 
\end{tabular} 
\end{table}

\paragraph{Time dependence under closed-loop control.}
Figure~\ref{fig:robotarm-time-sweep} shows that, under the closed-loop
robot-arm dynamics, the Hausdorff error grows only slowly with the simulation
horizon $T$. This contrasts with the linear-system example in
Section~\ref{subsec:lipschitz_sampling_bounds}, where flow expansion strongly amplifies sampling errors. In the robot-arm setting, the inverse-dynamics tracking
controller reduces the effective expansion of the uncertainty set, so the observed time dependence is much milder than the worst-case exponential growth predicted by the minimax bound.

\begin{figure}[htbp]
    \centering
    \includegraphics[width=0.6\linewidth]{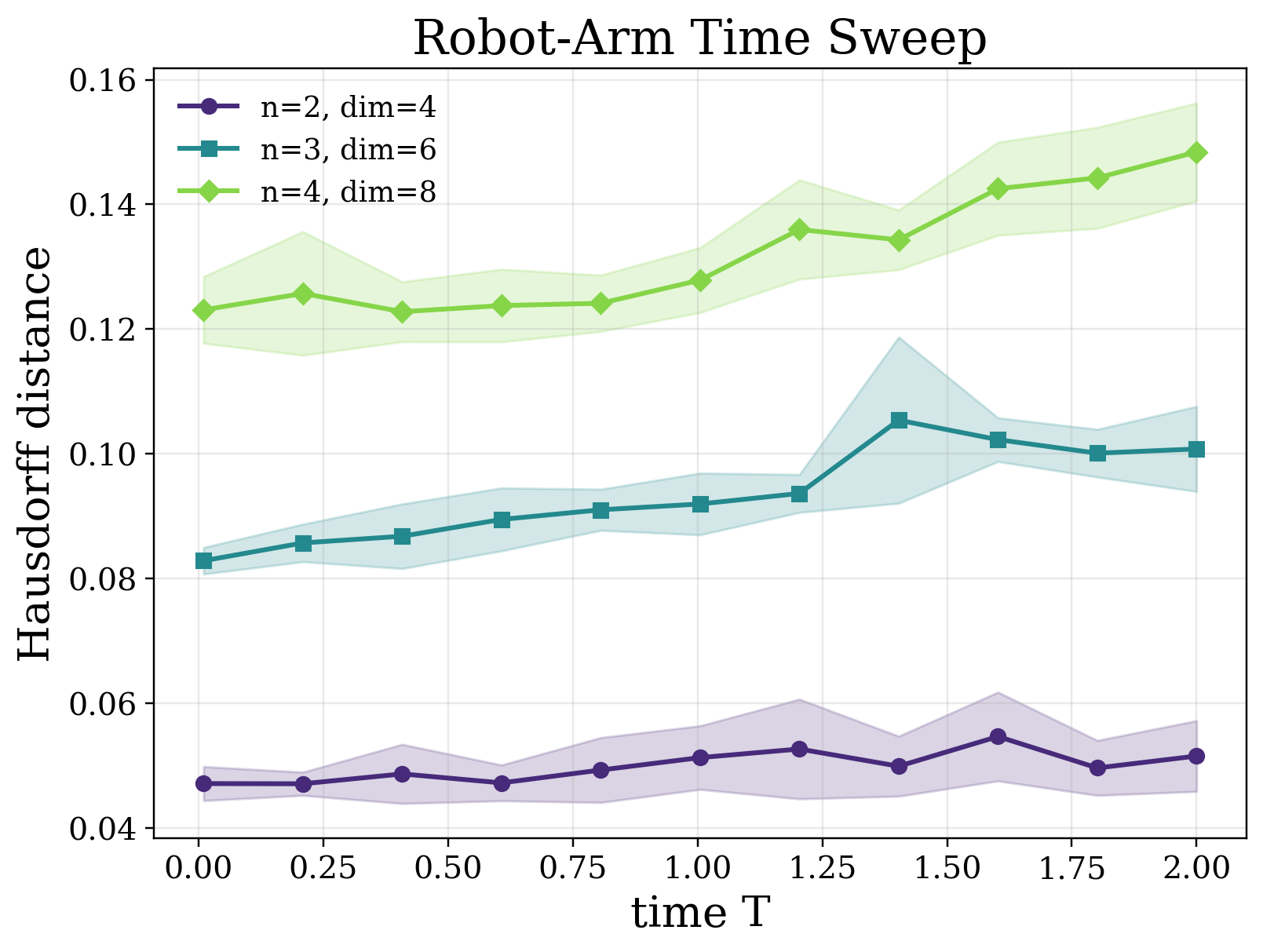}
    \caption{Hausdorff error versus time under uniform sampling.}
    \label{fig:robotarm-time-sweep}
\end{figure}

\paragraph{Effect of Adversarial Sampling Intensity Across Estimators.}

We further study how the adversarial intensity of the sampling procedure affects reachable-set approximation under different downstream estimators. We repeat the time-sweep experiment for the autonomous system $\dot y=0,\dot x=x^2$ using the same three initial sets as in Section~\ref{subsec:lipschitz_sampling_bounds}: a disk, a triangle, and an opened triangle with equal area. We compare two estimators constructed from the same endpoint samples: the convex hull estimator and a Christoffel-type estimator.

The adversarial intensity is controlled by the number of adversarial updates $n_{\rm adv}\in\{0,1,2,3,4\}$, where $n_{\rm adv}=0$ corresponds to purely uniform sampling. Larger $n_{\rm adv}$ allocates a larger fraction of the fixed sample budget to points obtained after adversarial updates, and therefore places more emphasis on expanding or boundary-like directions in the endpoint cloud.

Figures~\ref{fig:hausdorff_vs_time_uniform_adversarial_ci_convexhull} and~\ref{fig:hausdorff_vs_time_uniform_adversarial_ci_chris} show the Hausdorff
error over time for sample budgets $N=10,100,1000$. The results show that the effect of adversarial sampling is both budget-dependent and estimator-dependent. For very small sample budgets, increasing $n_{\rm adv}$ does not consistently improve performance, since overly concentrated adversarial samples may reduce global coverage. For larger budgets, adversarial updates become more beneficial:
the endpoint cloud already has enough global coverage, and additional adversarial samples help resolve boundary and extreme regions. This trend is visible for both convex-hull and Christoffel estimators, although the best choice of $n_{\rm adv}$ varies across estimators, sample budgets, and initial geometries. Overall, these results suggest a coverage--boundary tradeoff: stronger adversarial sampling can improve finite-sample performance, but only when the sample budget is large enough to avoid sacrificing global coverage.

\begin{figure}[htbp]
    \centering
    \begin{subfigure}{\linewidth}
        \includegraphics[width=\linewidth]{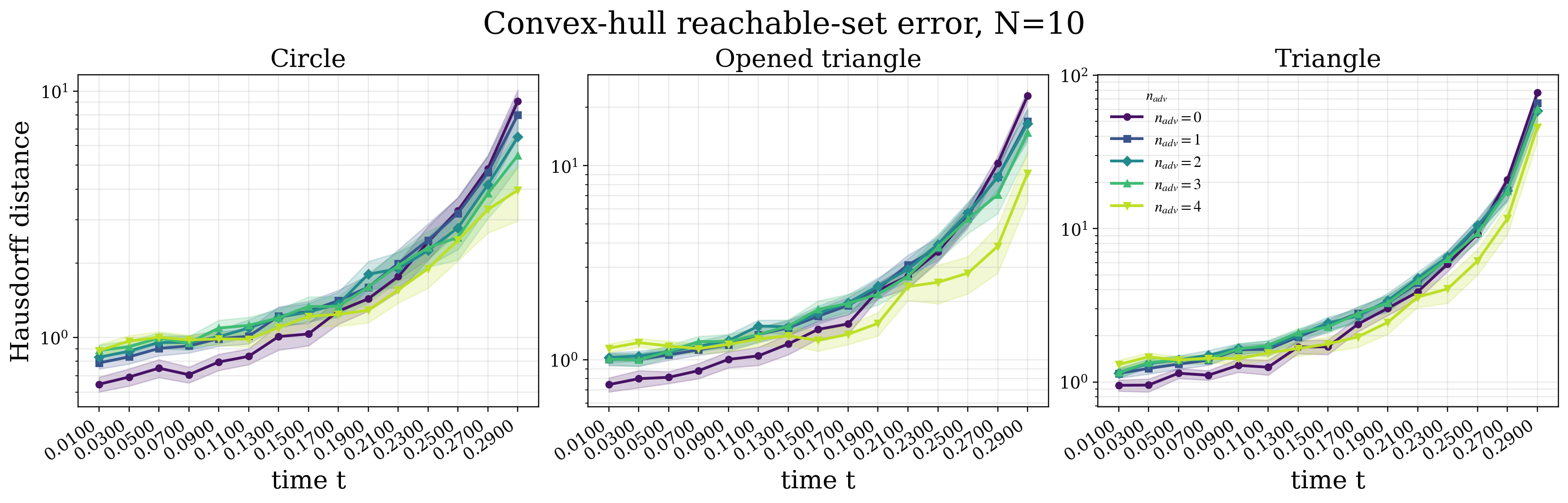}
        \label{fig:hausdorff_vs_time_convex_N10}
    \end{subfigure}

    \begin{subfigure}{\linewidth}
        \includegraphics[width=\linewidth]{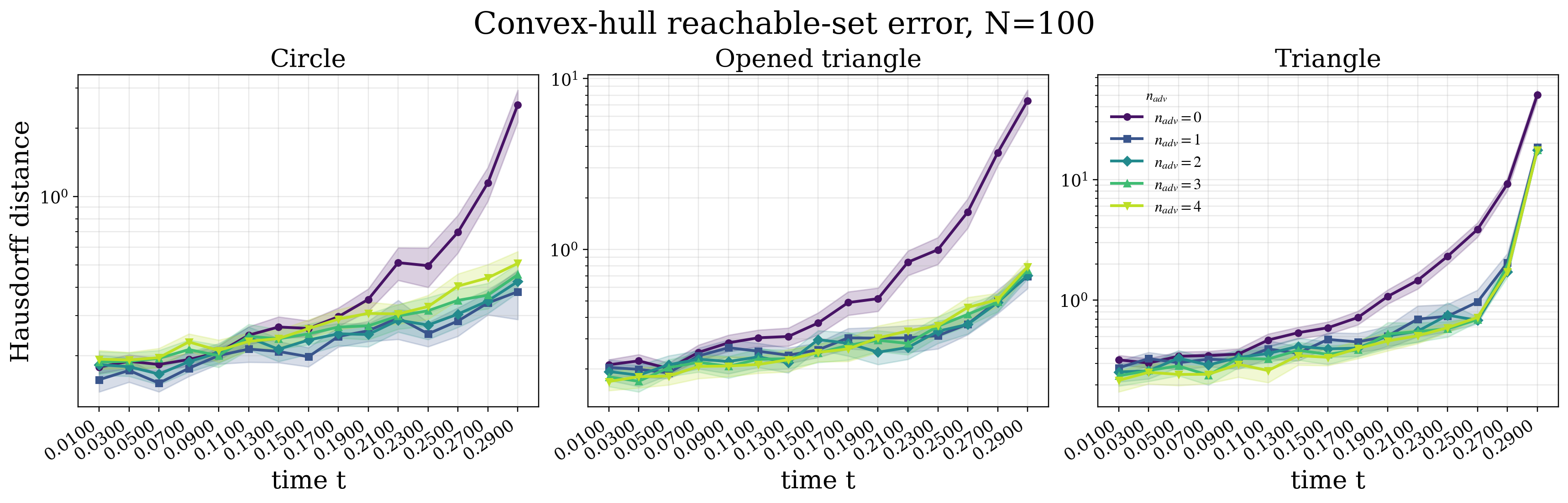}
        \label{fig:hausdorff_vs_time_convex_N100}
    \end{subfigure}

    \begin{subfigure}{\linewidth}
        \includegraphics[width=\linewidth]{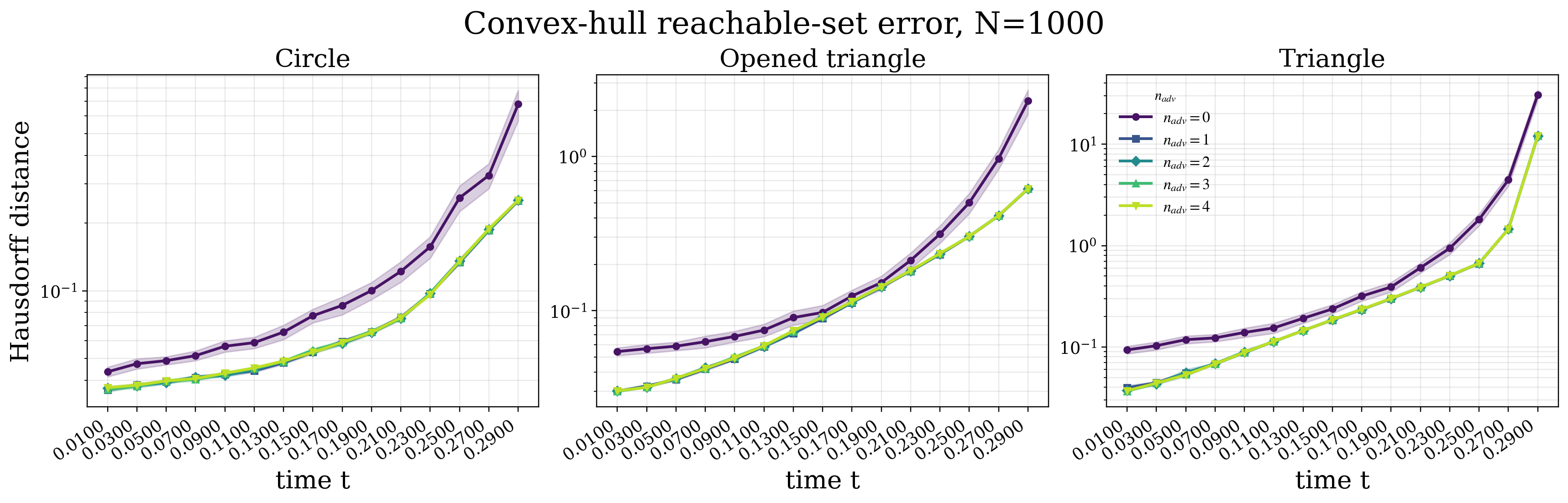}
        \label{fig:hausdorff_vs_time_convex_N1000}
    \end{subfigure}

    \caption{Hausdorff error versus time for convex-hull reachable-set approximation of the autonomous system $\dot y=0,\dot x=x^2$ under different uniform--adversarial mixture sampling ratios. The three subfigures correspond to sample budgets $N=10,100,1000$, and each subfigure compares the disk, triangle, and opened-triangle initial sets.}
    \label{fig:hausdorff_vs_time_uniform_adversarial_ci_convexhull}
\end{figure}

\begin{figure}[htbp]
    \centering
    \begin{subfigure}{\linewidth}
        \includegraphics[width=\linewidth]{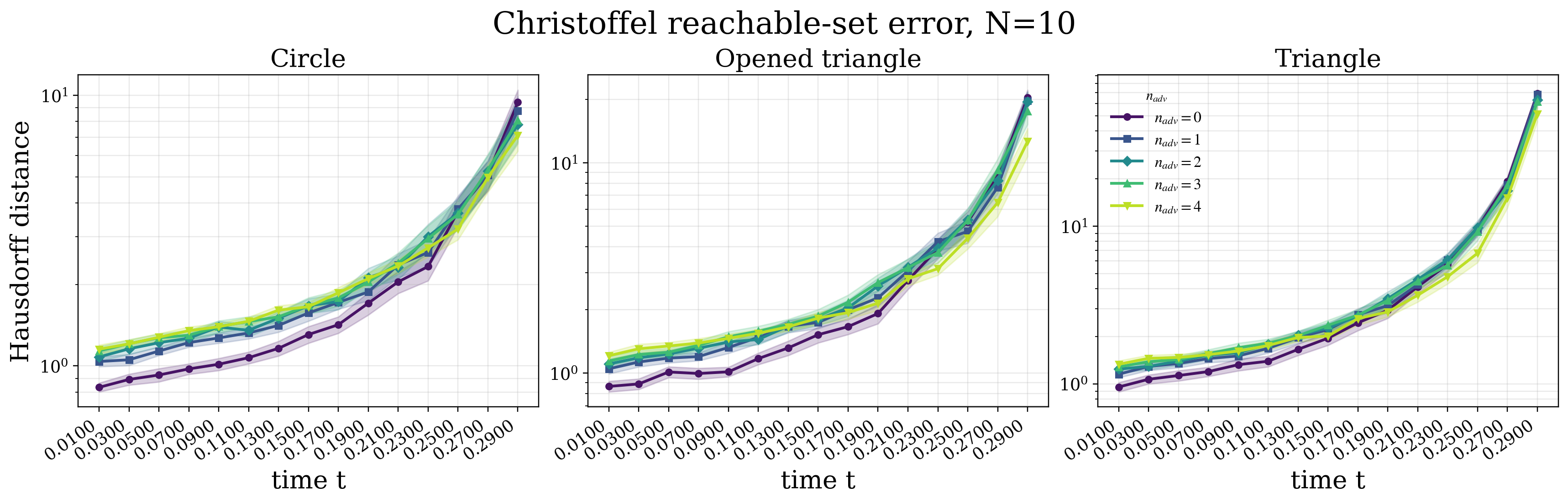}
        \label{fig:hausdorff_vs_time_chris_N10}
    \end{subfigure}

    \begin{subfigure}{\linewidth}
        \includegraphics[width=\linewidth]{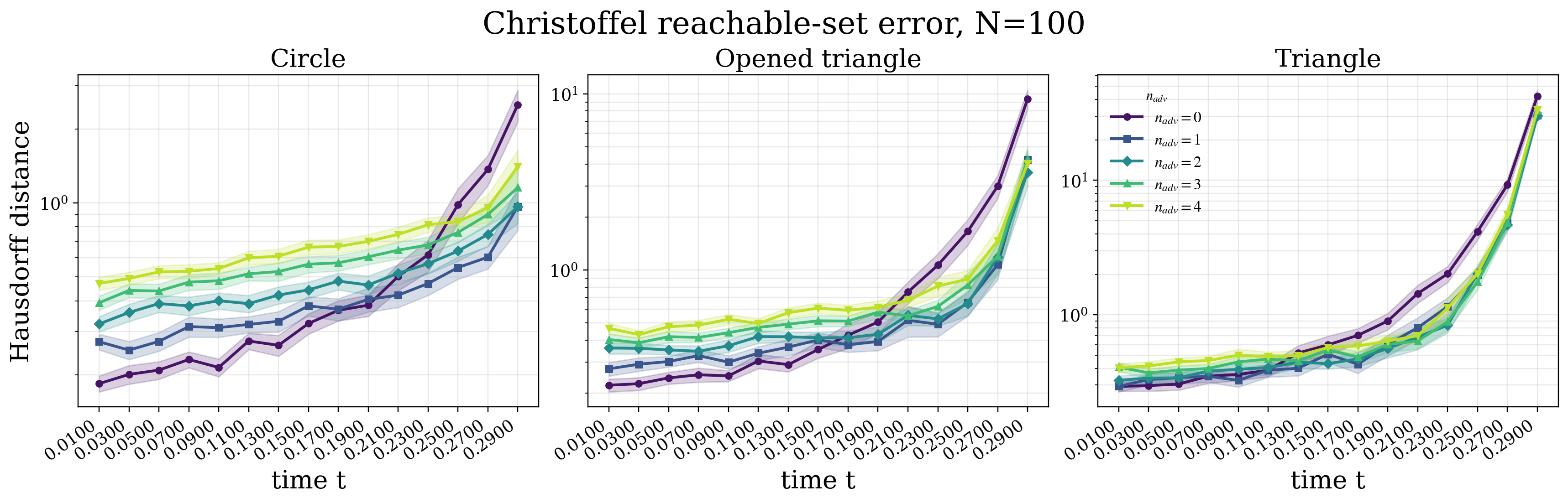}
        \label{fig:hausdorff_vs_time_chris_N100}
    \end{subfigure}

    \begin{subfigure}{\linewidth}
        \includegraphics[width=\linewidth]{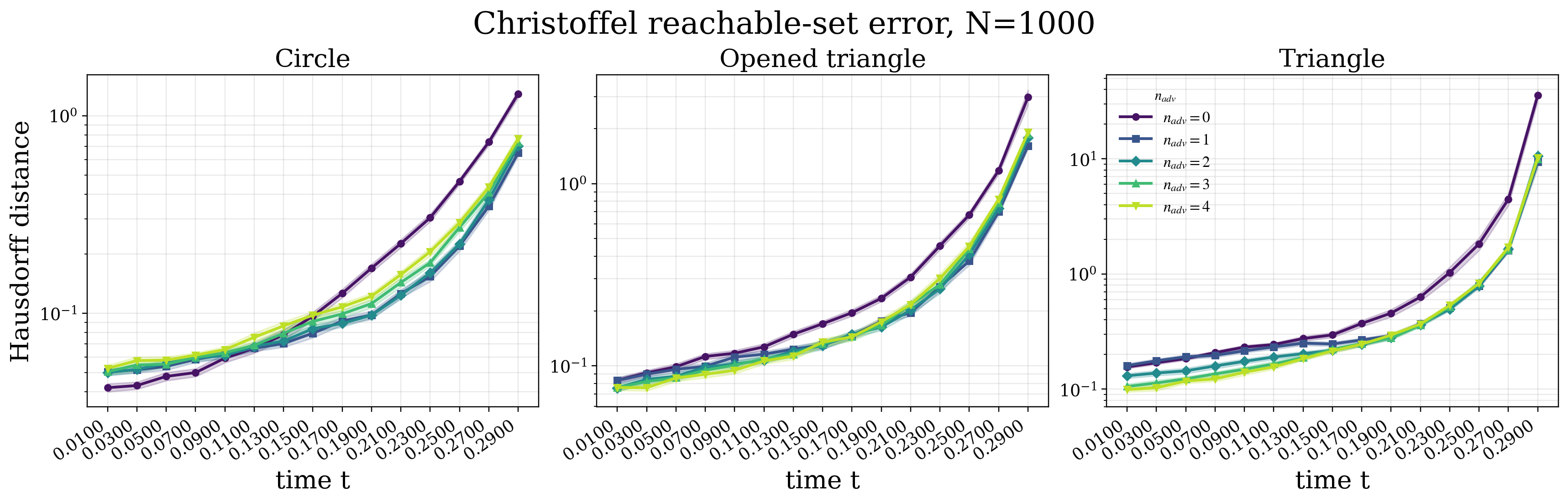}
        \label{fig:hausdorff_vs_time_chris_N1000}
    \end{subfigure}

    \caption{Hausdorff error versus time for Christoffel reachable-set approximation of the autonomous system $\dot y=0,\dot x=x^2$ under different uniform--adversarial mixture sampling ratios. The three subfigures correspond to sample budgets $N=10,100,1000$, and each subfigure compares the disk, triangle, and opened-triangle initial sets.}
    \label{fig:hausdorff_vs_time_uniform_adversarial_ci_chris}
\end{figure}
\end{document}